%% file: Mallows_CRC.tex
\documentclass{article} 
\usepackage{nips14submit_e,times}
\usepackage{url}

\usepackage{latexsym,xspace}
\usepackage{amsmath,amssymb,amsthm}
\nipsfinalcopy
\usepackage[small]{caption}
\usepackage{float}
\usepackage{balance}

\usepackage{color}

\title{Learning Mixtures of Ranking Models\thanks{This work was supported in part by NSF grants CCF-1101215, CCF-1116892, the Simons Institute, and a Simons Foundation Postdoctoral fellowhsip. Part of this work was performed while the 3rd author was at the Simons Institute for the Theory of Computing at the University of California, Berkeley and the 4th author was at CMU.}}
\author{Pranjal Awasthi\\
Princeton University\\
\texttt{pawashti@cs.princeton.edu}
\And Avrim Blum\\
Carnegie Mellon University\\
\texttt{avrim@cs.cmu.edu}
\And Or Sheffet\\
Harvard University\\
\texttt{osheffet@seas.harvard.edu}
\And
Aravindan Vijayaraghavan\\
New York University\\
\texttt{vijayara@cims.nyu.edu}}

\begin{document}

\maketitle


\newtheorem{theorem}{Theorem}[section]
\newtheorem{lemma}[theorem]{Lemma}
\newtheorem{definition}[theorem]{Definition}
\newtheorem{notation}[theorem]{Notation}
\newtheorem{property}[theorem]{Property}
\newtheorem{claim}[theorem]{Claim}
\newtheorem{proposition}[theorem]{Proposition}
\newtheorem{fact}[theorem]{Fact}
\newtheorem{remark}[theorem]{Remark}

\input{macros}
\input{abstract}
\section{Introduction}
\input{intro-pranjal}
\input{main_part_workshop}
\input{algorithm_overview}
\eat{
\section{Analysis of the Algorithm}
\label{sec:analysis}
\subsection{Retrieving Top Elements}
\label{sec:tensor}
\input{tensorpart}

\subsection{Recovering the Full Rankings}
\label{sec:full_recovery}
\input{fullrecovery}

\vspace{-10pt}
\subsection{Wrapping up the proof} \label{sec:wrapping}
\input{wrapping}
}
\section{Experiments}
\label{sec:experiments}
\input{experiments}

{\small
\bibliographystyle{unsrt}
\bibliography{mallows}
}
\newpage
\section{Acknowledgements}
We would like to thank Ariel Procaccia for bringing to our attention various references to Mallows model
in social choice theory.
\input{appendix}

\end{document}

%% file: macros.tex
\renewcommand{\P}{\ensuremath{\mathsf{P}}\xspace}
\newcommand{\NP}{\ensuremath{\mathsf{NP}}\xspace}
\newcommand{\NEXP}{\ensuremath{\mathsf{NEXP}}\xspace}
\newcommand{\coNP}{\ensuremath{\mathsf{coNP}}\xspace}
\newcommand{\coNEXP}{\ensuremath{\mathsf{coNEXP}}\xspace}
\renewcommand{\L}{\ensuremath{\mathsf{L}}\xspace}
\newcommand{\NL}{\ensuremath{\mathsf{NL}}\xspace}
\newcommand{\coNL}{\ensuremath{\mathsf{coNL}}\xspace}
\newcommand{\PSPACE}{\ensuremath{\mathsf{PSPACE}}\xspace}
\newcommand{\NC}{\ensuremath{\mathsf{NC}}\xspace}
\newcommand{\AC}{\ensuremath{\mathsf{AC}}\xspace}
\newcommand{\Ppoly}{\ensuremath{\mathsf{P/poly}}\xspace}
\newcommand{\BPP}{\ensuremath{\mathsf{BPP}}\xspace}
\newcommand{\RP}{\ensuremath{\mathsf{RP}}\xspace}
\newcommand{\ZPP}{\ensuremath{\mathsf{ZPP}}\xspace}
\newcommand{\coRP}{\ensuremath{\mathsf{coRP}}\xspace}
\newcommand{\EXP}{\ensuremath{\mathsf{EXP}}\xspace}
\newcommand{\IP}{\ensuremath{\mathsf{IP}}\xspace}
\newcommand{\PH}{\ensuremath{\mathsf{PH}}\xspace}
\newcommand{\polyL}{\ensuremath{\mathsf{polyL}}\xspace}
\newcommand{\RL}{\ensuremath{\mathsf{RL}}\xspace}
\newcommand{\PP}{\ensuremath{\mathsf{PP}}\xspace}
\newcommand{\MA}{\ensuremath{\mathsf{MA}}\xspace}
\newcommand{\AM}{\ensuremath{\mathsf{AM}}\xspace}
\renewcommand{\Pr}{{\bf Pr}} 
\newcommand{\E}{{\bf E}}
\newcommand{\Prx}{\mathop{\bf Pr\/}}         
\newcommand{\Ex}{\mathop{\bf E\/}}           
\newcommand{\Varx}{\mathop{\bf Var\/}}       
\newcommand{\Covx}{\mathop{\bf Cov\/}}       

\newcommand{\poly}{\mathrm{poly}}
\newcommand{\polylog}{\mathrm{polylog}}

\newcommand{\R}{\mathbb R}
\newcommand{\N}{\mathbb N}
\newcommand{\Z}{\mathbb Z}
\newcommand{\F}{\mathbb F}
\newcommand{\GF}{\mathrm{G}\mathbb{F}}

\newcommand{\bits}{\{0,1\}}
\newcommand{\bitsn}{\bits^n}

\newcommand{\eps}{\epsilon}
\newcommand{\lam}{\lambda}

\newcommand{\barx}{\overline{x}}
\newcommand{\bary}{\overline{y}}
\newcommand{\barz}{\overline{z}}

\newcommand{\calA}{{\cal A}}
\newcommand{\calB}{{\cal B}}
\newcommand{\calC}{{\cal C}}    
\newcommand {\calD}   {{\cal{D}}}
\newcommand {\calE}   {{\cal{E}}}
\newcommand {\calH}   {{\cal{H}}}
\newcommand {\calL}   {{\cal{L}}}
\newcommand {\calM}   {{\cal{M}}}
\newcommand {\calN}   {{\cal{N}}}
\newcommand {\calP}   {{\cal{P}}}
\newcommand {\calS}   {{\cal{S}}}
\newcommand {\calU}   {{\cal{U}}}
\newcommand {\calZ}   {{\cal{Z}}}

\newcommand{\bx}{{\boldsymbol{x}}}  
\newcommand{\by}{{\boldsymbol{y}}}  

\newcommand{\littlesum}{{\textstyle \sum}}              
\newcommand{\littlesumx}{\mathop{{\textstyle \sum}}}    
\newcommand{\half}{{\textstyle \frac12}}     

\newcommand{\AND}{\ensuremath{\mathsf{AND}}\xspace}
\newcommand{\OR}{\ensuremath{\mathsf{OR}}\xspace}
\newcommand{\XOR}{\ensuremath{\mathsf{XOR}}\xspace}
\newcommand{\NOT}{\ensuremath{\mathsf{NOT}}\xspace}
\newcommand{\MODp}{\ensuremath{\mathsf{MOD_p}}\xspace}

\newcommand{\SAT}{\ensuremath{\mathsf{SAT}}\xspace}
\newcommand{\Perm}{\ensuremath{\mathsf{Perm}}\xspace}

\newcommand{\OPT}{\ensuremath{\mathsf{OPT}}\xspace}
\newcommand{\SQ}{\ensuremath{\mathsf{SQ}}\xspace}
\newcommand{\err}{{\bf err}}
\newcommand{\freq}[1]{\text{freq}\left[ #1 \right]}

\newcommand {\compactMath}[1]{$#1$}

\newcommand {\nc}  {\frac{1}{\sqrt{2\pi}}}
\newcommand {\ncc} [2]  {\frac{#1}{\sqrt{2\pi} #2}}
\newcommand {\roundup}   [1] {{\lceil {#1} \rceil}}
\newcommand {\rounddown} [1] {{\lfloor {#1} \rfloor}}

\newcommand {\set}   [1] {\left\{ #1 \right\}}
\newcommand {\brc}   [1] {\left(#1\right)}

\newcommand {\Exp}       {\mathbb{E}}
\newcommand {\Prob}[1]{\Pr\brc{#1}}
\newcommand {\Probb}[2]{\Pr_{#1}\brc{#2}}
\newcommand {\EE}    [2] {\Exp_{#1}\left[#2\right]}
\newcommand {\Varr}  [1] {\Var \left[#1\right]}
\newcommand {\iprod} [2] {\langle #1, #2 \rangle}
\newcommand {\Iprod} [2] {\left\langle #1, #2 \right\rangle}
\newcommand {\tprod} [2] {{#1 \otimes #2}}
\newcommand{\given}{\mid}
\newcommand {\abs}  [1] {\lvert #1 \rvert}
\newcommand{\norm}[1] {\| #1 \|}
\newcommand{\TODO}[1] {{\bf TODO:} {\it #1 }}

\newcommand {\bbN}    {\mathbb{N}}
\newcommand {\bbZ}    {\mathbb{Z}}
\newcommand {\bbB}    {\mathbb{B}}
\newcommand {\bbR}    {\mathbb{R}}
\newcommand {\bbC}    {\mathbb{C}}

\newcommand{\Authornote}[2]{{\small{\sf$<<<${  #1: #2 }$>>>$}}}

\newcommand {\Zpart} [1] {Z_{#1}}
\newcommand {\Zn} {\Zpart{n}}
\newcommand {\Znn} {\Zpart{[n]}}
\newcommand{\Mal}[2]{\calM\left(#1 , #2 \right)}
\newcommand{\cMal}[2]{\widehat{\calM}\left(#1 , #2 \right)}
\newcommand{\MMix}{w_1\Mal{\phi_1}{\pi_1}\oplus w_2\Mal{\phi_2}{\pi_2}}
\newcommand{\MMixn}{w_1\calM_n(\phi_1,\pi_1)\oplus w_2\calM_n(\phi_2,\pi_2)}
\newcommand{\cMMix}{\cw_1\Mal{\cphi_1}{\pi_1}\oplus \cw_2\cMal{\cphi_2}{\pi_2}}
\newcommand{\MMixSame}{w_1\Mal{\phi}{\pi_1}\oplus w_2\Mal{\phi}{\pi_2}}
\newcommand{\el}[1]{e_{#1}}
\newcommand{\pos}[1]{\emph{pos}\left(\el{#1}\right)}
\newcommand{\posgen}{\emph{pos}}
\newcommand{\pospi}{\emph{pos}_\pi(\el{i})}
\newcommand{\delt}{\delta}
\newcommand{\tensors}{\textsc{Tensor-Decomp }}

\makeatletter
\floatstyle{ruled}
\newfloat{fragment}{H}{lop}
\floatname{fragment}{Algorithm}
\renewcommand{\floatc@ruled}[2]{\vspace{2pt}{\@fs@cfont \#1.\:} \#2 \par
 \vspace{1pt}}
\makeatother

\newcommand{\dkt}[1]{\ensuremath{d_{\mathsf{kt}}(#1)} }
\newcommand{\epsa}{\eps_1}
\newcommand{\epsc}{\eps_3}
\newcommand{\epsd}{\eps_4}
\newcommand{\epse}{\eps_5}
\newcommand{\epsw}{n^{-sc}}

\newcommand{\xa}{x^{(a)}}
\newcommand{\ya}{y^{(a)}}
\newcommand{\xb}{x^{(b)}}
\newcommand{\yb}{y^{(b)}}
\newcommand{\xc}{x^{(c)}}
\newcommand{\yc}{y^{(c)}}
\newcommand{\Ma}{M_a}
\newcommand{\Mb}{M_b}
\newcommand{\Mc}{M_c}

\newcommand{\ua}{u^{(a)}}
\newcommand{\va}{v^{(a)}}
\newcommand{\ub}{u^{(b)}}
\newcommand{\vb}{v^{(b)}}
\newcommand{\uc}{u^{(c)}}
\newcommand{\vc}{v^{(c)}}

\newcommand{\errp}{\vartheta}
\newcommand{\cphi}{\widehat{\phi}}
\newcommand{\cw}{\widehat{w}}
\newcommand{\cx}{\widehat{x}}
\newcommand{\cy}{\widehat{y}}
\newcommand{\minl}{\gamma_{\text{min}}}
\newcommand{\wmin}{w_{\text{min}}}
\newcommand{\phimin}{\phi_{\text{min}}}
\newcommand{\phimax}{\phi_{\text{max}}}
\newcommand{\llmin}{\minl \wmin}

\newcommand{\Maa}{M'_a}
\newcommand{\Mbb}{M'_b}
\newcommand{\Mcc}{M'_c}
\newcommand{\pa}{p^{(a)}}
\newcommand{\pb}{p^{(b)}}
\newcommand{\pc}{p^{(c)}}
\newcommand{\cpi}{\widehat{\pi}}
\newcommand{\cP}{\widehat{P}}


\newcommand{\epsb}{\eps'_2}
\newcommand{\epsta}{\eps'}
\newcommand{\epstaa}{\eps_6}

\newcommand{\cp}{\hat{p}}

\newcommand{\ccalM}{\widehat{\calM}}

\newcommand{\ftotal}[2]{f\left(#1 \rightarrow #2\right)}
\newcommand{\fone}[2]{f^{(1)}\left(#1 \rightarrow #2\right)}
\newcommand{\cfone}[3]{f^{(1)}\left(#1 \rightarrow #2\vert \el{#3} \rightarrow 1 \right)}
\newcommand{\cftwo}[3]{f^{(2)}\left(#1 \rightarrow #2 \vert \el{#3} \rightarrow 1 \right)}
\newcommand{\ftwo}[2]{f^{(2)} \left(#1 \rightarrow #2\right)}
\newcommand{\fsup}[3]{f^{(#3)}\left(#1 \rightarrow #2\right)}
\newcommand{\cfsup}[4]{f^{(#4)}\left(#1 \rightarrow #2\vert \el{#3} \rightarrow 1 \right)}
\newcommand{\fsub}[3]{f_{#3}\left(#1 \rightarrow #2\right)}

\newcommand{\hfone}[2]{{\hat f}^{(1)}\left(#1 \rightarrow #2\right)}
\newcommand{\hcfone}[3]{{\hat f}^{(1)}\left(#1 \rightarrow #2\vert \el{#3} \rightarrow 1 \right)}
\newcommand{\hcftwo}[3]{{\hat f}^{(2)}\left(#1 \rightarrow #2 \vert \el{#3} \rightarrow 1 \right)}
\newcommand{\hftwo}[2]{{\hat f}^{(2)} \left(#1 \rightarrow #2\right)}
\newcommand{\hf}[2]{{\hat f}\left(#1 \rightarrow #2\right)}
\newcommand{\chf}[3]{{\hat f}\left(#1 \rightarrow #2 \vert \el{#3} \rightarrow 1 \right)}

\newcommand{\yperp}{y^{\perp}}

%% file: abstract.tex
\begin{abstract}
This work concerns learning probabilistic models for ranking data in a heterogeneous population. The specific problem we study is learning the parameters of a {\em Mallows Mixture Model}. Despite being widely studied, current heuristics for this problem do not have theoretical guarantees and can get stuck in bad local optima. We present the first polynomial time algorithm which provably learns the parameters of a mixture of two Mallows models. A key component of our algorithm is a novel use of tensor decomposition techniques to learn the top-$k$ prefix in both the rankings. Before this work, even the question of {\em identifiability} in the case of a mixture of two Mallows models was unresolved.
\end{abstract}

%% file: intro-pranjal.tex
\newcommand{\eat}[1]{}

Probabilistic modeling of ranking data is an extensively studied problem with a rich body of past work~\cite{Mallow57,Marden95,LebanonL02,Murphy03,Meila07,Busse07,MandhaniM09,LuB11,Oren}. Ranking using such models has applications in a variety of areas ranging from understanding user preferences in  electoral systems and social choice theory, to more modern learning tasks in online web search, crowd-sourcing and recommendation systems. Traditionally, models for generating ranking data consider a homogeneous group of users with a \emph{central ranking} (permutation) $\pi^*$ over a set of $n$ elements or alternatives. (For instance, $\pi^*$ might correspond to a ``ground-truth ranking'' over a set of movies.) Each individual user generates her own ranking as a noisy version of this one central ranking and independently from other users. The most popular ranking model of choice is the \emph{Mallows model}~\cite{Mallow57}, where in addition to $\pi^*$ there is also a scaling parameter $\phi \in (0,1)$. Each user picks her ranking $\pi$ w.p. proportional to $\phi^{d_{\sf kt}(\pi,\pi^*)}$ where $d_{\sf kt}(\cdot)$ denotes the Kendall-Tau distance between permutations (see Section~\ref{sec:prelims}).\footnote{In fact, it was shown~\cite{Mallow57} that this model is the result of the following simple (inefficient) algorithm: rank every pair of elements randomly and independently s.t. with probability $\tfrac {1} {1+\phi}$ they agree with $\pi^*$ and with probability $\tfrac {\phi}{1+\phi}$ they don't;  if all $\binom n 2$ pairs agree on a single ranking -- output this ranking, otherwise resample.} We denote such a model as $\calM_n(\phi,\pi^*)$.

The Mallows model and its generalizations have received much attention from the statistics, political science and machine learning communities, relating this probabilistic model to the long-studied work about voting and social choice~\cite{Condorcet,dia88}. 
From a machine learning perspective, the problem is to find the parameters of the model~--- the central permutation $\pi^*$ and the scaling parameter $\phi$, using independent samples from the distribution. There is a large body of work~\cite{Murphy03, Busse07,Meila07, MandhaniM09, BravermanM09} providing efficient algorithms for learning the parameters of a Mallows model. 

\eat{
 In this work, we focus on the problem of learning a ranking model, and in particular, learning a Mallow's ranking model~\cite{Mallow57}. In the Mallow's model, all permutations are i.i.d samples from a particular distribution generated from the central permutation $\pi^*$, where the probability of sampling a specific permutation $\pi$ is proportional to $\exp(-\beta \dkt{\pi, \pi^*})$ where $\dkt{}$ denotes the Kendall-Tau distance between rankings (see below) and $\beta$ is the model's parameter. Our goal is to find the central permutation, denoted $\pi^*$ (and find $\beta$ as well), from multiple iid examples sampled from such distribution. This is motivated by a setting where we look at multiple rankings of (say) movies, with each individual has her own ranking over movies, and we postulate the existence of an overall movie ranking for the entire population s.t. each individual's ranking is derived (up to a few flips) from this one central ranking. Indeed, the Mallow model is quite popular and the problem of learning the Mallow model's parameters has been repeatedly studied in the past, see~\cite{Meila07, MandhaniM09, BravermanM09}.
}

In many scenarios, however, the population is heterogeneous with multiple groups of people, each with their own central ranking~\cite{Marden95}. For instance, when ranking movies, the population may be divided into two groups corresponding to men and women; with men ranking movies with one underlying central permutation, and women ranking movies with another underlying central permutation. This naturally motivates the problem of learning a \emph{mixture} of multiple Mallows models for rankings, a problem that has received significant attention~\cite{LuB11, MeilaC10, LebanonL02, Murphy03}. Heuristics like the EM algorithm have been applied to learn the model parameters of a mixture of Mallows models~\cite{LuB11}. The problem has also been studied under distributional assumptions over the parameters, e.g. weights derived from a Dirichlet distribution~\cite{MeilaC10}. However, unlike the case of a single Mallows model, algorithms with provable guarantees have remained elusive for this problem.

In this work we give the \emph{first polynomial time algorithm} that \emph{provably} learns a mixture of two Mallows models. 
The input to our algorithm consists of i.i.d random rankings (samples), with each ranking drawn with probability $w_1$ from a Mallows model $\calM_n(\phi_1,\pi_1)$, and with probability $w_2 (=1-w_1)$ from a different model $\calM_n(\phi_2,\pi_2)$.

\noindent\textbf{Informal Theorem.} {\it Given sufficiently many i.i.d samples drawn from a mixture of two Mallows models, we can learn the central permutations $\pi_1,\pi_2$ exactly and parameters $\phi_1,\phi_2,w_1,w_2$ up to $\epsilon$-accuracy in time $\poly (n,(\min\{w_1,w_2\})^{-1},\frac{1}{\phi_1 (1-\phi_1)},\frac{1}{\phi_2(1-\phi_2)},{\epsilon}^{-1})$.}

It is worth mentioning that, to the best of our knowledge, prior to this work even the question of identifiability was unresolved for a mixture of two Mallows models; given infinitely many i.i.d. samples generated from a mixture of two distinct Mallow models with parameters $\{w_1, \phi_1, \pi_1, w_2, \phi_2, \pi_2\}$ (with $\pi_1\neq \pi_2$ or $\phi_1\neq\phi_2$), could there be a different set of parameters $\{w'_1, \phi'_1, \pi'_1, w'_2, \phi'_2, \pi'_2\}$ which explains the data just as well.
Our result shows that this is not the case and the mixture is uniquely identifiable given polynomially many samples.


\noindent\textbf{Intuition and a Na\"ive First Attempt.} 
It is evident that having access to sufficiently many random samples allows one to learn a single Mallows model. Let the elements in the permutations be denoted as $\{e_1, e_2, \ldots, e_n\}$.
In a single Mallows model, the probability of element $\el{i}$ going to position $j$ (for $j\in [n]$) drops off exponentially as one goes farther from the true position of $\el{i}$~\cite{BravermanM09}. So by assigning each $e_i$ the most frequent position in our sample, we can find the central ranking $\pi^*$.

The above mentioned intuition suggests the following clustering based approach to learn a mixture of two Mallows models --- look at the distribution of the positions where element $\el{i}$ appears. If the distribution has 2 clearly separated ``peaks'' then they will correspond to the positions of $\el{i}$ in the central permutations. Now, dividing the samples according to $\el{i}$ being ranked in a high or a low position is likely to give us two pure (or almost pure) subsamples, each one coming from a single Mallows model. We can then learn the individual models separately. More generally, this strategy works when the two underlying permutations $\pi_1$ and $\pi_2$ are far apart which can be formulated as a {\em separation condition}.\footnote{Identifying a permutation $\pi$ over $n$ elements with a $n$-dimensional vector $(\pi(i))_i$, this separation condition can be roughly stated as $\|\pi_1-\pi_2\|_\infty = \tilde\Omega\left((\min\{w_1,w_2\})^{-1} \cdot (\min\{\log(1/\phi_1),\log(1/\phi_2)\}))^{-1}\right)$.} Indeed, the above-mentioned intuition works only under strong separator conditions: otherwise, the observation regarding the distribution of positions of element $\el{i}$ is no longer true \footnote{Much like how other mixture models are solvable under separation conditions, see~\cite{Dasgupta99, Arora01, AchlioptasM05}.}. For example, if $\pi_1$ ranks $\el{i}$ in position $k$ and $\pi_2$ ranks $\el{i}$ in position $k+2$, it is likely that the most frequent position of $\el{i}$ is $k+1$, which differs from $\el{i}$'s position in either permutations!

\noindent \textbf{Handling arbitrary permutations.}
Learning mixture models under no separation requirements is a challenging task. To the best of our knowledge, the only polynomial time algorithm known is for the case of a mixture of a constant number of Gaussians~\cite{KMV10,MV10}. Other works, like the recent developments that use tensor based methods for learning mixture models without  distance-based separation condition~\cite{AGHKT,AnandkumarHK12,Hsu13} still require non-degeneracy conditions and/or work for specific sub cases~(e.g. spherical Gaussians).

These sophisticated tensor methods form a key component in our algorithm for learning a mixture of two Mallows models. This is non-trivial as learning over rankings poses challenges which are not present in other widely studied problems such as mixture of Gaussians. For the case of Gaussians, spectral techniques
have been extremely successful~\cite{VempalaW04, AchlioptasM05, AGHKT, Hsu13}. Such techniques rely on estimating the covariances and higher order moments in terms of the model parameters to detect structure and dependencies. On the other hand, in the mixture of Mallows models problem there is \emph{no ``natural'' notion of a second/third moment}. A key contribution of our work is defining analogous notions of moments which can be represented succinctly in terms of the model parameters. As we later show, this allows us to use tensor based techniques to get a good starting solution.  

\noindent\textbf{Overview of Techniques.}
One key difficulty in arguing about the Mallows model is the lack of closed form expressions for basic propositions like \emph{``the probability that the $i$-th element of $\pi^*$ is ranked in position $j$.''} Our first observation is that the distribution of a given element appearing at the top, i.e. the first position, behaves nicely. Given an element $e$ whose rank in the central ranking $\pi^*$ is $i$, the probability that a ranking sampled from a Mallows model ranks $e$ as the first element is $\propto \phi^{i-1}$. A length $n$ vector consisting of these probabilities is what we define as the {\em first moment vector} of the Mallows model. Clearly by sorting the coordinate of the first moment vector, one can recover the underlying central permutation and estimate $\phi$. Going a step further, consider any two elements which are in positions $i,j$ respectively in $\pi^*$. We show that the probability that a ranking sampled from a Mallows model ranks $\{i,j\}$ in (any of the $2!$ possible ordering of) the first two positions is $\propto f(\phi)\phi^{i+j-2}$. We call the $n \times n$ matrix of these probabilities as the {\em second moment matrix} of the model~(analogous to the covariance matrix). Similarly, we define the $3$rd moment tensor as the probability that any $3$ elements appear in positions $\{1,2,3\}$. We show in the next section that in the case of a mixture of two Mallows models, the $3$rd moment tensor defined this way has a rank-$2$ decomposition, with each rank-$1$ term corresponds to the first moment vector of each of two Mallows models. This motivates us to use tensor-based techniques to estimate the first moment vectors of the two Mallows models, thus learning the models' parameters. 

The above mentioned strategy would work if one had access to infinitely many samples from the mixture model. 
But notice that the probabilities in the first-moment vectors decay exponentially, so by using polynomially many samples we can only recover a prefix of length $\sim {\log_{1/\phi} n}$ from both rankings. This forms the first part of our algorithm which outputs good estimates of the mixture weights, scaling parameters $\phi_1$, $\phi_2$ and prefixes of a certain size from both the rankings.
Armed with $w_1$, $w_2$ and these two prefixes we next proceed to recover the full permutations $\pi_1$ and $\pi_2$. In order to do this, we take two new fresh batches of samples. On the first batch, we estimate the probability that element $e$ appears in position $j$ for all $e$ and $j$. On the second batch, which is noticeably larger than the first, we estimate the probability that $e$ appears in position $j$ \emph{conditioned on a carefully chosen element $e^*$ appearing as the first element}. We show that this conditioning is {\em almost} equivalent to sampling from the same mixture model but with rescaled weights $w_1'$ and $w_2'$. The two estimations allow us to set a system of two linear equations in two variables: $\fone{e}{j}$ -- the probability of element $e$ appearing in position $j$ in $\pi_1$, and $\ftwo{e}{j}$ --- the same probability for $\pi_2$. Solving this linear system we find the position of $e$ in each permutation. 

The above description contains most of the core ideas involved in the algorithm. We need two additional components. First, notice that the $3$rd moment tensor is not well defined for triplets $(i,j,k)$, when $i,j,k$  are not all distinct and hence cannot be estimated from sampled data. To get around this barrier we consider a random partition of our element-set into $3$ disjoint subsets. The actual tensor we work with consists only of triplets $(i,j,k)$ where the indices belong to different partitions. Secondly, we have to handle the case where tensor based-technique fails, i.e. when the $3$rd moment tensor isn't full-rank. This is {\em a degenerate case}. Typically, tensor based approaches for other problems cannot handle such degenerate cases.
However, in the case of the Mallows mixture model, we show that such a degenerate case provides a lot of useful information about the problem. In particular, it must hold that $\phi_1 \simeq \phi_2$, \emph{and} $\pi_1$ and $\pi_2$ are fairly close --- one is almost a cyclic shift of the other. To show this we use a characterization of the when the tensor decomposition is unique (for tensors of rank $2$), and we handle such degenerate cases separately. Altogether, we find the mixture model's parameters with no non-degeneracy conditions.

\noindent {\bf Lower bound under the pairwise access model.} 
Given that a single Mallows model can be learned using only pairwise comparisons, a very restricted access to each sample, it is natural to ask, \emph{``Is it possible to learn a mixture of Mallows models from pairwise queries?''}. 
This next example shows that we cannot hope to do this even for a mixture of two Mallows models. Fix some $\phi$ and $\pi$ and assume our sample is taken using mixing weights of $w_1=w_2=\tfrac 1 2$ from the two Mallows models $\calM_n(\phi,\pi)$ and $\calM_n(\phi,\textrm{rev}(\pi))$, where $\textrm{rev}(\pi)$ indicates the reverse permutation (the first element of $\pi$ is the last of $\textrm{rev}(\pi)$, the second is the next-to-last, etc.) . Consider two elements, $e$ and $e'$. Using only pairwise comparisons, we have that it is just as likely to rank $e>e'$ as it is to rank $e'>e$ and so this case cannot be learned regardless of the sample size. 

\noindent\textbf{$3$-wise queries.}
We would also like to stress that our algorithm does not need full access to the sampled rankings and instead will work with access to certain $3$-wise queries. Observe that the first part of our algorithm, where we recover the top elements in each of the two central permutations, only uses access to the top $3$ elements in each sample. In that sense, we replace the pairwise query ``do you prefer $e$ to $e'$?'' with a $3$-wise query: ``what are your top $3$ choices?'' Furthermore, the second part of the algorithm (where we solve a set of $2$ linear equations) can be altered to support $3$-wise queries of the (admittedly, somewhat unnatural) form ``if $e^*$ is your top choice, do you prefer $e$ to $e'$?'' For ease of exposition, we will assume full-access to the sampled rankings.

\noindent\textbf{Future Directions.} Several interesting directions come out of this work. A natural next step is to generalize our results to learn a mixture of $k$ Mallows models for $k>2$. We believe that most of these techniques can be extended to design algorithms that take $\poly(n,1/\eps)^k$ time. It would also be interesting to get algorithms for learning a mixture of $k$ Mallows models which run in time $\poly(k,n)$, perhaps in an appropriate smoothed analysis setting~\cite{BCMV} or under other non-degeneracy assumptions. 
Perhaps, more importantly, our result indicates that tensor based methods which have been very popular for learning problems, might also be a powerful tool for tackling ranking-related problems in the fields of machine learning, voting and social choice.

\noindent\textbf{Organization.} In Section~\ref{sec:prelims} we give the formal definition of the Mallow model and of the problem statement, as well as some useful facts about the Mallow model. Our algorithm and its numerous subroutines are detailed in Section~\ref{sec:algorithm}. In Section~\ref{sec:experiments} we experimentally compare our algorithm with a popular EM based approach for the problem. The complete details of our algorithms and proofs are included in the supplementary material.

%% file: main_part_workshop.tex
\section{Notations and Properties of the Mallows Model} \label{sec:prelims}
Let $U_n=\{\el{1},\el{2},\dots,\el{n}\}$ be a set of $n$ distinct elements. We represent permutations over the elements in $U_n$ through their indices $[n]$. (E.g., $\pi=(n,n-1,\dots,1)$ represents the permutation $\left(\el{n},\el{n-1},\dots,\el{1}\right)$.)
Let $\pospi = \pi^{-1}(i)$ refer to the position of $\el{i}$ in the permutation $\pi$. We omit the subscript $\pi$ when the permutation $\pi$ is clear from context.
For any two permutations $\pi,\pi'$ we denote $\dkt{\pi,\pi'}$ as the Kendall-Tau distance~\cite{kendall38} between them (number of pairwise inversions between $\pi,\pi'$). Given some $\phi\in(0,1)$ we denote $\Zpart{i}(\phi) = \frac {1-\phi^i} {1-\phi}$, and partition function $\Znn(\phi) = \sum_\pi \phi ^{\dkt{\pi, \pi_0}}=  \prod_{i=1}^n \Zpart{i}(\phi)$ (see Section~\ref{apx_sec:properties_of_Mallow} in the supplementary material). 
\begin{definition}\label{def:mal}[Mallows model ($\calM_{n}(\phi, \pi_0)$).] Given a permutation $\pi_0$ on $[n]$ and a parameter $\phi \in (0,1)$,\footnote{It is also common to parameterize using $\beta \in \R^{+}$ where $\phi = e^{-\beta}$. For small $\beta$ we have $(1-\phi)\approx \beta$.}, a \emph{Mallows model} is a permutation generation process that returns permutation $\pi$ w.p. \[\Prob{\pi}=\phi^{\dkt{\pi, \pi_0}}/ \Znn(\phi)\]
\end{definition} 

In Section~\ref{apx_sec:properties_of_Mallow} we show many useful properties of the Mallows model which we use repeatedly throughout this work. We believe that they provide an insight to Mallows model, and we advise the reader to go through them. We proceed with the main definition.

\begin{definition}\label{def:mmix}[Mallows Mixture model $\MMixn$.] 
Given parameters $w_1,w_2\in (0,1)$ s.t. $w_1+w_2=1$, parameters $\phi_1,\phi_2\in(0,1)$ and two permutations $\pi_1, \pi_2$, we call a \emph{mixture of two Mallows models} to be the process that with probability $w_1$ generates a permutation from $\Mal{\phi_1}{\pi_1}$  and with probability $w_2$ generates a permutation from $\Mal{\phi_2}{\pi_2}$. 
\end{definition}

Our next definition is crucial for our application of tensor decomposition techniques.

\begin{definition}\label{def:xy}[Representative vectors.]
The \emph{representative vector} of a Mallows model is a vector where for every $i\in[n]$, the $i$th-coordinate is $\phi^{\pospi-1}/\Zpart{n}$.\end{definition} 
\vspace{-0.3cm}
The expression $\phi^{\pospi-1}/\Zpart{n}$ is precisely the probability that a permutation generated by a model $\calM_n(\phi,\pi)$ ranks element $e_i$ at the first position (proof deferred to the supplementary material). Given that our focus is on learning a mixture of two Mallows models $\calM_n(\phi_1,\pi_1)$ and $\calM_n(\phi_2,\pi_2)$, we denote $x$ as the representative vector of the first model, and $y$ as the representative vector of the latter. Note that retrieving the vectors $x$ and $y$ exactly implies that we can learn the permutations $\pi_1$ and $\pi_2$ and the values of $\phi_1,\phi_2$.

Finally, let $\ftotal{i}{j}$ be the probability that element $\el{i}$ goes to position $j$ according to mixture model. Similarly $\fone{i}{j}$ be the corresponding probabilities according to Mallows model $\calM_1$ and $\calM_2$ respectively. Hence, $\ftotal{i}{j}=w_1 \fone{i}{j}+w_2 \ftwo{i}{j}$.  

\noindent \textbf{Tensors:} Given two vectors $u \in \R^{n_1},  v \in \R^{n_2}$, we define $u \otimes v \in R^{n_1 \times n_2}$ as the matrix $uv^T$. Given also $z\in\R^{n_3}$ then $u\otimes v\otimes z$ denotes the $3$-tensor (of rank-
$1$) whose $(i,j,k)$-th coordinate is $u_iv_jz_k$. A tensor $T \in \R^{n_1 \times n_2\times n_3}$ has a rank-$r$ decomposition if $T$ can be expressed as $\sum_{i \in [r]} u_i \otimes v_i \otimes z_i$ where $u_i \in \R^{n_1}, v_i \in \R^{n_2}, z_i \in \R^{n_3}$. Given two vectors $u, v \in \R^n$, we use $\left(u ; v \right)$ to denote the $n \times 2$ matrix that is obtained with $u$ and $v$ as columns. 

We now define first, second and third order statistics (frequencies) that serve as our proxies for the first, second and third order moments. 
\begin{definition}\label{def:xy_moments}[Moments]
Given a Mallows mixture model, we denote for every $i,j,k \in [n]$ 
\begin{itemize}
\item $P_i=\Prob{\pos{i}=1}$ is the probability that element $\el{i}$ is ranked at the first position
\item $P_{ij}=\Prob{\posgen\left(\set{\el{i},\el{j}}\right)=\{1,2\}}$, is the probability that $\el{i},\el{j}$ are ranked at the first two positions (in any order)
\item $P_{ijk}=\Prob{\posgen\left(\set{\el{i},\el{j},\el{k}}\right)=\{1,2,3\}}$ is the probability that $\el{i}, \el{j}, \el{k}$ are ranked at the first three positions (in any order).
\end{itemize}
\end{definition}
\noindent For convenience, let $P$ represent the set of quantities $\left( P_i, P_{ij}, P_{ijk} \right)_{1 \le i < j <k \le n}$. These can be estimated up to any inverse polynomial accuracy using only polynomial samples. 
The following simple, yet crucial lemma relates $P$ to the vectors $x$ and $y$, and demonstrates why these statistics and representative vectors are ideal for tensor decomposition.
\begin{lemma}\label{lem:moments}
Given a mixture $\MMix$ let $x, y$ and $P$ be as defined above. 
\begin{enumerate}
\item For any $i$ it holds that $P_i=  w_1 x_i + w_2 y_i.$
\item Denote $c_2(\phi)=\frac{\Zpart{n}(\phi)}{\Zpart{n-1}(\phi)}\frac{1+\phi}{\phi}$. Then for any $i\neq j$ it holds that  
$P_{ij}=w_1 c_2(\phi_1)x_i x_j + w_2 c_2(\phi_2) y_i y_j$.
\item Denote $c_3(\phi)=\frac{\Zpart{n}^2(\phi)}{\Zpart{n-1}(\phi)\Zpart{n-2}(\phi)}\frac{1+2\phi+2\phi^2+\phi^3}{\phi^3}$. Then for any distinct $i,j,k$ it holds that 
$P_{ijk}= w_1 c_3(\phi_1) x_i x_j x_k+ w_2 c_3(\phi_2)y_i y_j y_k$.
\end{enumerate}
Clearly, if $i=j$ then $P_{ij}=0$, and if $i,j,k$ are not all distinct then $P_{ijk}=0$.
\end{lemma}
\noindent In addition, in Lemma~\ref{lem:boundingc} in the supplementary material we prove the bounds $c_2(\phi) = O(1/\phi)$ and $c_3(\phi)= O(\phi^{-3})$.

\noindent\textbf{Partitioning Indices:} Given a partition of $[n]$ into $S_a, S_b, S_c$, let $\xa,\ya$ be the representative vectors $x, y$ restricted to the indices (rows) in $S_a$ (similarly for $S_b, S_c$). 
Then the $3$-tensor $$T^{(abc)} \equiv (P_{ijk})_{\substack{i \in S_a, j \in S_b, k \in S_c}}= w_1 c_3(\phi_1)\xa \otimes \xb \otimes \xc + w_2 c_3(\phi_2) \ya \otimes \yb \otimes \yc.$$
This tensor has a rank-$2$ decomposition, with one rank-$1$ term for each Mallows model. Finally for convenience we define the matrix $M=(x ; y)$, and similarly define the matrices $\Ma=(\xa ; \ya)$, $\Mb=(\xb ; \yb)$, $\Mc=(\xc ; \yc)$.

{\bf Error Dependency and Error Polynomials.} Our algorithm gives an estimate of the parameters $w,\phi$ that we learn in the first stage, and we use these estimates to figure out the entire central rankings in the second stage. The following lemma essentially allows us to assume instead of estimations, we have access to the true values of $w$ and $\phi$.
\begin{lemma}
\label{lem:simulate-noisy-oracle}
For every $\delta > 0$ there exists a function $f(n,\phi,\delta)$ s.t. for every $n$, $\phi$ and $\hat\phi$ satisfying $|\phi - \hat\phi| < \frac {\delta}{f(n,\phi,\delta)}$ we have that the total-variation distance satisfies $\|\Mal{\phi}{\pi}-\Mal{\hat\phi}{\pi}\|_{\sf TV}\leq\delta$.
\end{lemma}
For the ease of presentation, we do not optimize constants or polynomial factors in all parameters. In our analysis, we show how our algorithm is robust (in a polynomial sense) to errors in various statistics, to prove that we can learn with polynomial samples. However, the simplification when there are no errors (infinite samples) still carries many of the main ideas in the algorithm --- this in fact shows the identifiability of the model, which was not known previously.  

%% file: algorithm_overview.tex
\section{Algorithm Overview}\label{sec:algorithm}

\begin{fragment*}[h]
\caption{\label{alg:main}{\sc Learn Mixtures of two Mallows models}, \textbf{Input: } a set $\cal{S}$ of $N$ samples from $\MMix$, Accuracy parameters $\epsilon, \epsilon_2$.  \vspace*{0.01in}
}
\begin{enumerate} \itemsep 0pt
\small 
\item Let $\cP$ be the empirical estimate of $P$ on samples in $\cal{S}$. 
\item Repeat $O(\log n)$ times:
\begin{enumerate}
\item Partition $[n]$ randomly into $S_a$, $S_b$ and $S_c$. Let $T^{(abc)} = \big ( \cP_{ijk} \big )_{i \in S_a, j \in S_b, k \in S_c}$.
\item Run \tensors from \cite{BCV,GVX,BCMV} to get a decomposition of $T^{(abc)} = \ua \otimes \ub \otimes \uc + \va \otimes \vb \otimes \vc$.
\item If $\min\{\sigma_2(\ua ; \va), \sigma_2 (\ub ; \vb) , \sigma_2 (\uc ; \vc)\}> \epsilon_2$ \\(In the {\em non-degenerate} case these matrices are far from being rank-$1$ matrices in the sense that their least singular value is bounded away from $0$.)
\begin{enumerate}
\item Obtain parameter estimates $(\cw_1, \cw_2, \cphi_1, \cphi_2$ and prefixes of the central rankings $ {\pi_1}^{\prime}, {\pi_2}^{\prime})$ from {\sc Infer-Top-k($\cP$, $\Maa $, $\Mbb$, $\Mcc$)}, with $M'_i = (u ^{(i)}; v^{(i)})$ for $i \in \{a,b,c\}$.
\item Use {\sc Recover-Rest} to find the full central rankings $\cpi_1, \cpi_2$.\\
Return {\sc Success} and output $(\cw_1, \cw_2, \cphi_1, \cphi_2, \cpi_1, \cpi_2)$.
\end{enumerate}   
\end{enumerate}
\item Run {\sc Handle Degenerate Cases ($\cP$)}.
\end{enumerate}
\end{fragment*}

Our algorithm (Algorithm~\ref{alg:main}) has two main components. First we invoke a decomposition algorithm~\cite{BCV,GVX,BCMV} over the tensor $T^{(abc)}$, and retrieve approximations of the two Mallows models' representative vectors which in turn allow us to approximate the weight parameters $w_1,w_2$, scale parameters $\phi_1$, $\phi_2$, and the top few elements in each central ranking. We then use the inferred parameters to recover the entire rankings $\pi_1$ and $\pi_2$. Should the tensor-decomposition fail, we invoke a special procedure to handle such degenerate cases. Our algorithm has the following guarantee.

\begin{theorem}\label{thm:main}
Let $\MMix$ be a mixture of two Mallows models and let $w_{\min} = \min\{w_1, w_2\}$ and $\phi_{\max} = \max\{\phi_1,\phi_2\}$ and similarly $\phi_{\min} = \min\{\phi_1,\phi_2\}$. Denote $\eps_0 = \frac{\wmin^2 (1-\phi_{\max})^{10}}{16n^{22}\phi^2_{\max}}$.
Then, given any $0<\eps<\eps_0$, suitably small $\eps_2 = \poly(\frac 1 n,\eps,\phimin,\wmin)$ 
 and $N = \poly\left(n,\frac{1}{\min\set{\eps,\eps_0}},\frac{1}{\phi_1(1-\phi_1)},\frac{1}{\phi_2(1-\phi_2)}, \frac{1}{w_1}, \frac{1}{w_2}\right)$ i.i.d samples from the mixture model, Algorithm~\ref{alg:main} recovers, in poly-time and with probability $\geq 1 - n^{-3}$, the model's parameters with $w_1,w_2,\phi_1,\phi_2$ recovered up to $\epsilon$-accuracy. 
\end{theorem}
\vspace{-0.3cm}

Next we detail the various subroutines of the algorithm, and give an overview of the analysis for each subroutine. The full analysis is given in the supplementary material.

\noindent\textbf{The \tensors Procedure.} This procedure is a straight-forward invocation of the algorithm detailed in~\cite{BCV,GVX,BCMV}. This algorithm uses spectral methods to retrieve the two vectors generating the rank-$2$ tensor $T^{(abc)}$. This technique works when all factor matrices $\Ma=(\xa ; \ya),\Mb=(\xb ; \yb),\Mc=(\xc ; \yc)$ are well-conditioned. We note that any algorithm that decomposes non-symmetric tensors which have well-conditioned factor matrices, can be used as a black box.

\begin{lemma}[Full rank case] \label{lem:fullrank}
In the conditions of Theorem~\ref{thm:main}, suppose our algorithm picks some partition $S_a, S_b, S_c$ such that the matrices $\Ma,\Mb,\Mc$ are all well-conditioned --- 
i.e. have $\sigma_2 (\Ma), \sigma_2 (\Mb), \sigma_2 (\Mc) \ge \epsb \ge \poly(\frac{1}{n}, \epsilon, \eps_2, w_1,w_2)$
then with high probability, Algorithm {\sc TensorDecomp} of~\cite{BCV} finds $\Maa=(\ua; \va), \Mbb=(\ub; \vb), \Mcc=(\uc; \vc)$  such that
for any $\tau \in \set{a, b, c}$, we have $u^{(\tau)} = \alpha_\tau x^{(\tau)}+ z_1^{(\tau)}$ and $v^{(\tau)} = \beta_\tau y^{(\tau)}+ z_2^{(\tau)}$; with $\norm{z_1^{(\tau)}},\norm{z_2^{(\tau)}} \le \poly(\frac 1 n, \epsilon, \eps_2, \wmin)$ and, 
$\sigma_2(M'_\tau)>\eps_2$ for $\tau \in \set{a,b,c}$.
\end{lemma}
\eat{In practice, since \tensors only relies on elements that have a non-negligible chance of appearing in the first position: this can lead to large speedup for constant $\phi < 1$ by restricting to a much smaller tensor. We refer the reader to Section~\ref{app:tensors} in the supplementary material for full details.}

\noindent\textbf{The \textsc{Infer-Top-k} procedure.}
This procedure uses the output of the tensor-decomposition to retrieve the weights, $\phi$'s and the representative vectors. In order to convert $\ua,\ub,\uc$ into an approximation of $\xa,\xb,\xc$ (and similarly with $\va,\vb,\vc$ and $\ya,\yb,\yc)$, we need to find a good approximation of the scalars $\alpha_a,\alpha_b,\alpha_c$. This is done by solving a certain linear system. This also allows us to estimate $\cw_1,\cw_2$. Given our approximation of $x$, it is easy to find $\phi_1$ and the top first elements of $\pi_1$ --- we sort the coordinates of $x$, setting $\pi_1'$ to be the first elements in the sorted vector, and $\phi_1$ as the ratio between any two adjacent entries in the sorted vector. We refer the reader to Section~\ref{app:tensors} in the supplementary material for full details.
\noindent\textbf{The \textsc{Recover-Rest} procedure.}
The algorithm for recovering the remaining entries of the central permutations (Algorithm~\ref{alg:recover-rest}) is more involved. 
\begin{fragment*}[ht]
\caption{\label{alg:recover-rest}{\sc Recover-Rest}, \textbf{Input: } a set $\cal{S}$ of $N$ samples from $\MMix$, parameters $\hat{w_1}, \hat{w_2}, \hat{\phi_1}, \hat{\phi_2}$ and initial permutations $\hat{\pi_1}, \hat{\pi_2}$, and accuracy parameter $\epsilon$.  \vspace*{0.01in}
}
\begin{enumerate} \itemsep 0pt
\small
\item For elements in $\hat{\pi_1}$ and $\hat{\pi_2}$, compute representative vectors $\hat{x}$ and $\hat{y}$ using estimates $\hat{\phi_1}$ and $\hat{\phi_2}$.
\item Let $|\hat{\pi_{1}}|= r_1$, $|\hat{\pi_{2}}| = r_2$ and wlog $r_1 \ge  r_2$.\\
If there exists an element $\el{i}$ such that $\posgen_{\hat{\pi}_1}(\el{i}) > r_1$ and $\posgen_{\hat{\pi}_2}(\el{i}) < r_2/2$ (or in the symmetric case), then: \\ Let $\mathcal{S}_1$ be the subsample with $\el{i}$ ranked in the first position.
\begin{enumerate}
\item Learn a single Mallows model on $\mathcal{S}_1$ to find $\hat{\pi_1}$. Given $\hat{\pi_1}$ use dynamic programming to find $\hat{\pi_2}$
\end{enumerate}
\item Let $e_{i^*}$ be the first element in $\hat{\pi_1}$ 
having its probabilities of appearing in first place in $\pi_1$ and $\pi_2$ differ by at least $\epsilon$.
Define $\hat{w}^{\prime}_1 = \left(1 + \tfrac{\hat{w_2}}{\hat{w_1}}\tfrac{\hat{y}({e_{i^*}})}{\hat{x}({e_{i^*}})}\right)^{-1}$ and $\hat{w}^{\prime}_2 = 1 - \hat{w}^{\prime}_1$. Let $\mathcal{S}_1$ be the subsample with $e_{i^*}$ ranked at the first position.
\label{alg:recover-rest-find-xstar}
\item For each $e_i$ that doesn't appear in either $\hat\pi_1$ or $\hat\pi_2$ and any possible position $j$ it might belong to
\label{alg:recover-rest-equations}
\begin{enumerate}
\item Use $\mathcal{S}$ to estimate $\hat{f}_{i,j} = \Prob{\el{i}\textrm{ goes to position }j}$, and $\mathcal{S}_1$ to estimate $\chf{i}{j}{i^*} = \Prob{\el{i}\textrm{ goes to position }j | e_{i^*} \mapsto 1}$.
\item Solve the system 
\begin{eqnarray}
\hf{i}{j} & = & \hat{w_1}\fone{i}{j} + \hat{w_2}\ftwo{i}{j}\\
\chf{i}{j}{i^*} & = & \hat{w}'_1\fone{i}{j} + \hat{w}'_2\ftwo{i}{j}
\end{eqnarray}
\end{enumerate}
\item To complete $\hat\pi_1$ assign each $\el{i}$ to position $\arg\max_{j}\{ \fone{i}{j}\}$. Similarly complete $\hat\pi_2$ using $\ftwo{i}{j}$. Return the two permutations. 
\end{enumerate}
\end{fragment*}

Algorithm~\ref{alg:recover-rest} first attempts to find a pivot --- an element $\el{i}$ which appears at a fairly high rank in one permutation, yet does not appear in the other prefix $\hat{\pi_2}$. Let $E_{\el{i}}$ be the event that a permutation ranks $\el{i}$ at the first position. As $\el{i}$ is a pivot, then $\Probb{\calM_1}{E_{\el{i}}}$ is noticeable whereas $\Probb{\calM_2}{E_{\el{i}}}$ is negligible. Hence, conditioning on $\el{i}$ appearing at the first position leaves us with a subsample in which all sampled rankings are generated from the first model. This subsample allows us to easily retrieve the rest of $\pi_1$. Given $\pi_1$, the rest of $\pi_2$ can be recovered using a dynamic programming procedure. Refer to the supplementary material for details.

The more interesting case is when no such pivot exists, i.e., when the two prefixes of $\pi_1$ and $\pi_2$ contain almost the same elements. Yet, since we  invoke \textsc{Recover-Rest} after successfully calling \tensors, it must hold that the distance between the obtained representative vectors $\hat{x}$ and $\hat{y}$ is noticeably large. Hence some element $\el{i^*}$ satisfies $|\hat{x}(\el{i^*}) - \hat{y}(\el{i^*}) |>\eps$, and we proceed by setting up a linear system.
To find the complete rankings, we measure appropriate statistics to set up a system of linear equations to calculate $\fone{i}{j}$ and $\ftwo{i}{j}$ up to inverse polynomial accuracy. The largest of these values $\set{\fone{i}{j}}$ corresponds to the position of $\el{i}$ in the central ranking of $\calM_1$. 

To compute the values $\set{\fsup{i}{j}{r}}_{r=1,2}$ we consider $\cfone{i}{j}{i^*}$ -- the probability that $\el{i}$ is ranked at the $j$th position conditioned on the element $\el{i^*}$ ranking first  according to $\calM_1$ (and resp. for $\calM_2$). Using $w'_1$ and $w'_2$ as in Algorithm~\ref{alg:recover-rest}, it holds that 
$$\Prob{\el{i} \rightarrow j \vert \el{i^*} \rightarrow 1} = w'_1 \cfone{i}{j}{i^*}+ w'_2 \cftwo{i}{j}{i^*}.$$

We need to relate $\cfsup{i}{j}{i^*}{r}$ to $\fsup{i}{j}{r}$. Indeed Lemma~\ref{lem:linear-equations-correct} shows that $\Prob{\el{i} \rightarrow j \vert \el{i^*} \rightarrow 1}$ is an \emph{almost} linear equations in the two unknowns. We show that if $\el{i^*}$ is ranked above $\el{i}$ in the central permutation, then for some small $\delta$ it holds that
\[\Prob{\el{i} \rightarrow j \vert \el{i^*} \rightarrow 1}= w'_1 \fone{i}{j}+ w'_2 \ftwo{i}{j} \pm \delta \]
We refer the reader to Section~\ref{app:recover-rest} in the supplementary material for full details.

\noindent\textbf{The \textsc{Handle-Degenerate-Cases} procedure.} 
We call a mixture model $\MMix$ \emph{degenerate} if the parameters of the two Mallows models are equal, and the edit distance between the prefixes of the two central rankings is at most two i.e., by changing the positions of at most two elements in $\pi_1$ we retrieve $\pi_2$.
We show that unless $\MMix$ is degenerate, a random partition $(S_a,S_b,S_c)$ is likely to satisfy the requirements of Lemma~\ref{lem:fullrank} (and \tensors will be successful). Hence, if \tensors repeatedly fail, we deduce our model is indeed degenerate. To show this, we characterize the uniqueness of decompositions of rank $2$, along with some very useful properties of random partitions. In such degenerate cases, we find the two prefixes and then remove the elements in the prefixes from $U$, and recurse on the remaining elements. We refer the reader to Section~\ref{app:degenerate_cases} in the supplementary material for full details.

\eat{
Our algorithm (Algorithm~\ref{alg:main}) has two main components. In the first part we use spectral methods to recover elements which have a good chance of appearing in the first position.  
We show that if the tensor $T^{(abc)}$ as defined in step (b) of the algorithm has a unique rank 2 decomposition, then one can recover the top few elements of both $\pi_1$ and $\pi_2$ correctly. In addition, we can also infer the parameters $w$'s and $\phi$'s to good accuracy~(corresponding to {\sc Infer-Top-k}). 

The second part of the algorithm~(corresponding to {\sc Recover-Rest}) takes the inferred parameters and the initial prefixes as input and uses this information to recover the entire rankings $\pi_1$ and $\pi_2$. This is done by observing that the probability of an element $e_i$ going to position $j$ can be written as a weighted combination of the corresponding probabilities under $\pi_1$ and $\pi_2$. In addition, as mentioned in Section~\ref{sec:prelims}, the reduced distribution obtained by conditioning on a particular element $e_j$ going to position $1$ is again a mixture of two Mallows models with the same parameters. Hence, by conditioning on a particular element which appears in the initial learned prefix, we get a system of linear equations which can be used to infer the probability of every other element $e_i$ going to position $j$ in both $\pi_1$ and $\pi_2$. This will allow us to infer the entire rankings.

\noindent \textbf{Degenerate Cases:}\\
The above discussion, however, assumes that the tensor $T^{(abc)}$ has a unique rank 2 decomposition. For spectral methods, such non-degeneracy assumptions are needed for provable performance. Hence, finally, we need an additional subroutine~({\sc Handle-Degenerate-Case}) to solve for degenerate cases. 

We call a mixture model $\MMix$ as degenerate if the parameters of the two Mallows models are equal, and the edit distance between the two central rankings is at most two (please see section~\ref{sec:degenerate} for more details. 
 Intuitively, in this case one of the partitions $S_a, S_b, S_c$ constructed by the algorithm will have their corresponding $u$ and $v$ vectors as parallel to each other and hence the tensor method will fail. We show that when this happens, it can be detected and in fact this case provides useful information about the model parameters. More specifically, we show that in a degenerate case, $\phi_1$ will be almost equal to $\phi_2$ and the two rankings will be aligned in a couple of very special configurations~(see Section~\ref{subsec:degenerate}). Procedure {\sc Handle-Degenerate-Case} is designed to recover the rankings in such scenarios. Our algorithms culminate in the following guarantee

The error polynomial $\errp_{\ref{thm:main}}=\errp_{\ref{lem:tensoralg}}(n,\errp_{\ref{lem:degenerate}}(n,\phimin,\epsilon), \phimin, \wmin)$ will suffice for this algorithm. Since all the polynomials $\errp_x$ are only polynomials in the various parameters, this gives a polynomial time algorithm for learning the Mallows model mixture.

}

%% file: tensorpart.tex
Here we show how the first stage of the algorithm i.e. {\em steps (a)-(e.i)} manages to recover the top few elements of both rankings $\pi_1$ and $\pi_2$ and also estimate the parameters $\phi_1,\phi_2, w_1, w_2$. We first show that if $\Ma, \Mb, \Mc$ have non-negligible minimum singular values, then the decomposition is unique, and hence we can recover the top few elements and parameters from {\sc Infer Top-$k$}. Otherwise, we show that if this procedure did not work for all $O(\log n)$ iterations, we are in \emph{the degenerate case} (Lemma~\ref{lem:eqphi} and Lemma~\ref{lem:bucket2}), and handle this separately.

The following Lemma captures how Algorithm~\ref{alg:main} (steps 3 (a - e.i)) performs the first stage using Algorithm~\ref{alg:inferparams} and recovers the weights $w_1,w_2$ and $x,y$ when the factor matrices $\Ma,\Mb,\Mc$ are well-conditioned.    

\begin{lemma}[Full rank case] \label{lem:fullrank}
In the conditions of Theorem~\ref{thm:main}, suppose there exists some partition $S_a, S_b, S_c$ such that the matrices $(\xa ; \ya),(\xb ; \yb)$ and $(\xc ; \yc)$ are well-conditioned i.e. have $\sigma_2 (\cdot) \ge \epsb \ge \frac{\eps_2}{\minl}-\errp_{\ref{lem:tensoralg}}$, then {\em Step (e).i.} of Algorithm~\ref{alg:main} finds vectors $\cx, \cy$ (upto renaming) and corresponding weights $\cw_1, \cw_2, \cphi_1,\cphi_2$ within error $\eps= \errp_{\ref{lem:fullrank}}\left(n,\eps'_2,\eps_s,\wmin,\phimin\right)$.
\end{lemma}
In practice, since \tensors only relies on elements that have a non-negligible chance of appearing in the first position: this can lead to large speedup for constant $\phi < 1$ by restricting to a much smaller tensor. 

In the proof we show that in this case, for one of the $O(\log n)$ random partitions, Lemma~\ref{lem:tensoralg} succeeds and recovers vectors $\ua, \va$ which are essentially parallel to $\xa$ and $\ya$ respectively (similarly for $\ub,\uc,\vb,\vc$). Sorting the entries of $\ua$ would give the relative ordering among those in $S_a$ of the top few elements of $\pi_1$. However, to figure out all the top-$k$ elements, we need to figure out the correct scaling of $\ua, \ub, \uc$ to obtain $\xa$. This is done by setting up a linear system. The proof details are in the Appendix~\ref{app:tensors}.

The above lemma shows that we succeed when $\Ma,\Mb,\Mc$ have non-negligible minimum singular value for one of the the $O(\log n)$ random partitions. If it is the case that for {\em all} of these $O(\log n)$ random partitions the tensor method {\em fails}, we show that we are in a very special case. The parameters of the two models $\phi_1$ and $\phi_2$ are essentially the same (Lemma~\ref{lem:eqphi}). Further, we can use this to find this parameter as well. Then we show that the two central rankings have an edit distance of at most two i.e. the two central rankings are the same up to a couple of shifts (see Lemma~\ref{lem:bucket2}). We then handle this case separately in Appendix~\ref{app:degenerate_cases}.




%% file: fullrecovery.tex
Let $\fone{i}{j}$ be the probability that element $\el{i}$ goes to position $j$ according to Mallows Model $\calM_1$ (and similarly $\ftwo{i}{j}$ for model $\calM_2$). 
To find the complete rankings, we measure appropriate statistics to set up a system of linear equations to calculate $\fone{i}{j}$ and $\ftwo{i}{j}$ up to inverse polynomial accuracy. The largest of these values $\set{\fone{i}{j}}$ corresponds to the position of $\el{i}$ in the central ranking of $\Mal_1$. To compute these values $\set{\fsup{i}{j}{r}}_{r=1,2}$ we consider statistics of the form ``\emph{what is the probability that $\el{i}$ goes to position $j$ conditioned on $\el{i^*}$ going to the first position}?''. This statistic is related to $\fone{i}{j}, \ftwo{i}{j}$, for an element $\el{i^*}$ that is much closer than $\el{i}$ to the front of one of the permutations. \\

 Suppose $\cfone{i}{j}{i^*}$ be the probability that $\el{i}$ goes to the $j$th position conditioned on the element $\el{i^*}$ going to the first position according to $\calM_1$ (similarly $\Mal_2$).  We have that for any elements $\el{i^*}, \el{i}$ and position $j$, we have 
\begin{align*}
\Prob{\el{i} \rightarrow j \vert \el{i^*} \rightarrow 1}&= w'_1 \cfone{i}{j}{i^*}+ w'_2 \cftwo{i}{j}{i^*} \\
\text{where } w'_1&=\frac{w_1 x_{i*}}{w_1 x_{i^*}+w_2 y_{i^*}}, w'_2=1-w'_1
\end{align*} 

However, these statistics are not in terms of the unknown variables $\fone{i}{j}, \ftwo{i}{j}$. Lemma~\ref{lem:linear-equations-correct} shows that these statistics are \emph{almost} linear equations in the required unknowns (for some small $\delta$):
$$\Prob{\el{i} \rightarrow j \vert \el{i^*} \rightarrow 1}= w'_1 \fone{i}{j}+ w'_2 \ftwo{i}{j} \pm \delta $$

Hence, by picking an appropriate element $\el{i^*}$, we can set up a system of linear equations and solves for the quantities $\set{\fone{i}{j}, \ftwo{i}{j}}$. Suppose there exists an element $\el{i^*}$ that occurs in the top few positions in both the permutations, then that element would suffice for our purpose. On the other hand, if we condition on an element $i^*$ which occurs near the top in one permutation but far away in the other permutation, gives us a \emph{single} Mallows model. The sub-routine {\sc Recover-Rest} of the main algorithm figures out which of the cases we are in, and succeeds in recovering the entire permutations $\pi_1$ and $\pi_2$. Please see section~\ref{app:recover-rest} in the full version for details. 

\begin{fragment*}[t]
\caption{\label{alg:recover-rest}{\sc Recover-Rest}, \textbf{Input: } a set $\cal{S}$ of $N$ samples from $\MMix$, $\hat{w_1}, \hat{w_2}, \hat{\phi_1}, \hat{\phi_2}, \hat{\pi_1}, \hat{\pi_2}, \epsilon$.  \vspace*{0.01in}
}
\begin{enumerate} \itemsep 0pt
\small
\item Let $|\hat{\pi_{1}}|= r_1$, $|\hat{\pi_{2}}| = r_2$ and $r_1 \ge  r_2$.
\item For any element $\el{i}$, define $\hat{p}_{i,1} = \frac{\hat{\phi_1}^{\left(\hat{\pi_1}^{-1}(e_i)-1\right)}}{Z_n(\hat{\phi_1})}$, and $\hat{q}_{i,1} = \frac{\hat{\phi_2}^{\left(\hat{\pi_2}^{-1}(e_i)-1\right)}}{Z_n(\hat{\phi_2})}$. If $e_i$ does not appear in $\hat{\pi_1}$ set $\hat{p}_{i,1} = 0$. Similarly, if $e_i$ does not appear in $\hat{\pi_2}$ set $\hat{q}_{i,1} = 0$. 
\item If there exists an element $\el{i}$ such that $\posgen_{\hat{\pi}_1}(i) > r_1$ and $\posgen_{\hat{\pi}_2}(i) < r_2$ (or in the symmetric case), then \\
Run {\sc Learn-Single-Mallow} to find $\hat{\pi_1}$. Then run {\sc Find-Pi} to find the other central permutation and  {\sc Return} .
\item If not, let $e_{i^*}$ be the first element in $\hat{\pi_1}$ such that $|\hat{p}_{i^*,1} - \hat{q}_{i^*,1}| > \epsilon$. Define $\hat{w}^{\prime}_1 = \frac 1 {1 + \frac{\hat{w_2}}{\hat{w_1}}\frac{\hat{q}_{i^*,1}}{\hat{p}_{i^*,1}}}$ and $\hat{w}^{\prime}_2 = 1 - \hat{w}^{\prime}_1$.
\label{alg:recover-rest-find-xstar}
\item For each $e_i \notin \hat{\pi_{1}}$ and $j > r$
\label{alg:recover-rest-equations}
\begin{enumerate}
\item Estimate $\hat{f}_{i,j} = Pr[\textrm{$e_i$ goes to position $j$}]$ and $\hat{f}_{e_i,j | x \mapsto 1} = Pr[\textrm{$e_i$ goes to position $j$} | e_{i^*} \mapsto 1]$.
\item Solve the system 
\begin{eqnarray}
\hat{f}_{i,j} & = & \fone{i}{j} + \hat{w_2}\ftwo{i}{j}\\
\hat{f}_{i,j | e_{i^*} \mapsto 1} & = & \hat{w}'_1\cfone{i}{j}{i^*} + \hat{w}'_2\cftwo{i}{j}{i^*}
\end{eqnarray}
\end{enumerate}
\item Form the ranking $\hat{\pi_1} = \hat{\pi_{1}} \circ \pi'_1$ s.t. for each $a_i \notin \hat{\pi_1} $, $pos(a_i) = argmax_{j > r} \hat{p}_{a_i,j}$.
\end{enumerate}
\end{fragment*}

%% file: wrapping.tex
\eat{
\begin{theorem}\label{thm:main}
For any $0<\eps$, suppose we are given $N$ samples from a mixture of two Mallows models $\MMix$ on $n \ge 6$ elements with mixing weights $w_1, w_2 \ge \eps_0$, there is an polynomial time algorithm (Algorithm~\ref{alg:main}) that recovers with high probability, the permutations $\pi_1, \pi_2$ and the parameters $w_1, w_2, \phi$ up to $\eps$-accuracy, as long as $\eps< \eps_0\cdot \poly\left(\frac{1-\phi}{n}\right)$ and $N \ge \poly\left(n,1/\eps,\frac{1}{(1-\phi)}\right)$. 
\end{theorem}
}

Let $\eps_s$ be the entry-wise error in $P$ from the estimates. From Lemma~\ref{lem:sampling}, $\eps_s < 3\log n/\sqrt{N}$. We aim to estimate each of the parameters $\phi_1,\phi_2 w_1, w_2$ up to error at most $\eps$. Let for convenience, $\gamma=\frac{(1-\phi_{\max})^2}{4n\phi_{\max}}$. 

Let $\eps=\min\set{\eps,\eps_0}$. Let $\eps_3=\errp_{\ref{lem:degenerate}}(n,\phimin,\eps)$.  Let us also set $\eps_2=\errp_{\ref{lem:eqphi}}(n,\eps_3,\phimin,\wmin)$. Let $\eps'_2$ be a parameter chosen large enough such that $\eps'_2 \ge \frac{\eps_2}{\minl}+\errp_{\ref{lem:tensoralg}}$, and $\eps \le \errp_{\ref{lem:fullrank}}(n,\eps'_2,\eps_s,\wmin,\phimin)$. 

In the non-degenerate case, suppose there is a partition such that $\sigma_2(\Ma), \sigma_2(\Mb),\sigma_2(\Mc) \ge \eps'_2$, Lemma~\ref{lem:tensoralg} guarantees that $\sigma_2(\Maa),\sigma_2(\Mbb),\sigma_2(\Mcc) \ge \eps_2$. In this case, Lemma~\ref{lem:fullrank} ensures that one of the $O(\log n)$ rounds of the algorithm succeeds and we get the parameters $w_1, w_2, \phi_1, \phi_2$ within an error $\eps$ using Lemma~\ref{lem:fullrank}. Further, Lemma~\ref{lem:fullrank} will also find the top $r,s$ elements of $\pi_1$ and $\pi_2$ respectively where $r=\log_{1/\phi_1}\left(\frac{n^{10}}{\gamma^2 \wmin}\right)$ and $s=\log_{1/\phi_2}\left(\frac{n^{10}}{\gamma^2 \wmin}\right)$. We will then appeal to Lemma~\ref{lem:recover-rest} (along with Lemma~\ref{lem:simulate-noisy-oracle}) to recover the entire rankings $\pi_1, \pi_2$.

Lemma~\ref{lem:simulate-noisy-oracle} implies that the total variation distance between distributions of $\MMix$ and $\cMMix$ is at most $\frac{\epsilon n^2}{\phimin}$. Since $\epsilon \le \eps_0$, this variation distance is at most $\frac{\phimin}{10n^3 S(\wmin/2, \sqrt{\phimax})}$. Here $S(\wmin/2,\sqrt{\phimax})$ is an upper bound on the number of samples needed by {\sc Recover-Rest} to work given true parameters (and not estimations). This allows to analyze the performance of {\sc Recover-Rest} assuming that we get perfect estimates of the parameters $(w_1, w_2,\phi_1, \phi_2)$ since samples used by {\sc Recover-Rest} which are drawn from $\MMix$ will be indistinguishable from samples from $\cMMix$ except with probability $\frac{1}{10n^3}$. This followed by the guarantee of Lemma~\ref{lem:recover-rest} will recover the complete rankings $\pi_1$ and $\pi_2$. 

In the degenerate case, due to our choice of $\eps_2$, Lemma~\ref{lem:eqphi} shows that $\cphi$ is $\eps_3$ close to both $\phi_1$ and $\phi_2$. Using Lemma~\ref{lem:degenerate} we then conclude that step 4 of Algorithm~\ref{alg:main} recovers $\pi_1,\pi_2$ and the parameters $w_1,w_2$ within error $\eps$.

%% file: experiments.tex
\vspace{-3mm}
\noindent\textbf{Goal.} The main contribution of our paper is devising an algorithm that \emph{provably} learns any mixture of two Mallows models. But could it be the case that the previously existing heuristics, even though they are unproven, still perform well in practice? 
We compare our algorithm to existing techniques, to see if, and under what settings our algorithm outperforms them.

\noindent\textbf{Baseline.} We compare our algorithm to the popular EM based algorithm of~\cite{Meila07}, seeing as EM based heuristics are the most popular way to learn a mixture of Mallows models. The EM algorithm starts with a random guess for the two central permutations.
At iteration $t$, EM maintains a guess as to the two Mallows models that generated the sample. First (expectation step) the algorithm assigns a weight to each ranking in our sample, where the weight of a ranking reflects the probability that it was generated from the first or the second of the current Mallows models. Then (the maximization step) the algorithm updates its guess of the models' parameters based on a local search -- minimizing the average distance to the weighted rankings in our sample. We comment that we implemented only the version of our algorithm that handles non-degenerate cases~(more interesting case). In our experiment the two Mallows models had parameters $\phi_1 \neq \phi_2$, so our setting was never degenerate.

\noindent\textbf{Setting.} We ran both the algorithms on synthetic data comprising of rankings of size $n=10$. The weights were sampled u.a.r from $[0,1]$, and the $\phi$-parameters were sampled by sampling $\ln(1/\phi)$ u.a.r from $[0,5]$. For $d$ ranging from $0$ to $\binom n 2$ we generated the two central rankings $\pi_1$ and $\pi_2$ to be within distance $d$ in the following manner. $\pi_1$ was always fixed as $(1,2,3,\ldots,10)$. To describe $\pi_2$, observe that it suffices to note the number of inversion between $1$ and elements $2,3,...,10$; the number of inversions between $2$ and $3,4,...,10$ and so on. So we picked u.a.r a non-negative integral solution to $x_1 + \ldots + x_n = d$ which yields a feasible permutation and let $\pi_2$ be the permutation that it details.
Using these models' parameters, we generated $N=5\cdot 10^{6}$ random samples.

\noindent\textbf{Evaluation Metric and Results.} For each value of $d$, we ran both algorithms $20$ times and counted the fraction of times on which they returned the true rankings that generated the sample.
The results of the experiment for rankings of size $n=10$ are in Table~\ref{tab:exp_results}. Clearly, the closer the two centrals rankings are to one another, the worst EM performs. On the other hand, our algorithm is able to recover the true rankings even at very close distances. As the rankings get slightly farther, our algorithm recovers the true rankings all the time. We comment that similar performance was observed for other values of $n$ as well. We also comment that our algorithm's runtime was reasonable (less than $10$ minutes on a $8$-cores Intel x86\_ 64 computer). Surprisingly, our implementation of the EM algorithm typically took much longer to run --- due to the fact that it simply did not converge.
\vspace{-3mm}
\begin{table}[h]
\centering\small
\begin{tabular}{|c|c|c|}
\hline
distance between rankings & success rate of EM & success rate of our algorithm\\
\hline
0 & 0\% & 10\%\\
\hline
2 & 0\% & 10\% \\
\hline
4 & 0\% & 40\% \\
\hline
8 & 10\% & 70\% \\
\hline
16 & 30\% & 60 \%\\
\hline
24 & 30\% & 100\%\\
\hline
30 & 60\% & 100\% \\
\hline
35 & 60\% &100\% \\
\hline
40 & 80\% &100\%\\
\hline
45 & 60\% & 100\%\\
\hline
\end{tabular}
\caption{\label{tab:exp_results} Results of our experiment.}
\end{table}


%

%% file: appendix.tex
\section{Properties of the Mallows Model}
\label{apx_sec:properties_of_Mallow}

In this section, we outline some of the properties of the Mallows model. Some of these properties were already shown before (see~\cite{Stanley02}), but we add them in this appendix for completion. Our algorithm and its analysis rely heavily on these properties.

\noindent\textbf{Notation.}
Given a Mallows model $\calM_n\left(\phi,\pi_0\right)$ we denote $\Zn = \frac {1-\phi^n} {1-\phi}$, and we denote $\Znn$ as the sum all weights of all permutations: $\Znn = \sum_{\pi} \phi^{d_{\sf kt}(\pi,\pi_0)}$. Given an element $e$, we abuse notation and denote by $\pi\setminus e$ the permutation we get by omitting the element $e$ (projecting $\pi$ over all elements but $e$). The notation $\pi=(e,\sigma)$ denotes a permutation whose first element is $e$ and elements $2$ through $n$ are as given by the permutation over $n-1$ elements $\sigma$.

The first property shows that for any element $e$, conditioning on $e$ being ranked at the first position results in a reduced Mallows model. 
\begin{lemma}
\label{app_lem:conditioned-on-x}
Let $\Mal{\phi}{\pi}$ be a Mallows model over $[n]$. For any $i$, the conditional distribution~(given that $i$ is ranked at position $1$) of rankings over $[n] \setminus \{i\}$, i.e. $\Prob{\pi | \pi(i)=1}$ is the same as that of $\Mal{\phi}{\pi \setminus {i}}$.
\end{lemma}
The above lemma can be extended to conditioning on prefixes as follows. 
\begin{lemma}
\label{app_lem:ignore-prefix}
Let $\Mal{\phi}{\pi}$ be a Mallows model over $[n]$. For any prefix $I$ of $\pi$, the marginal distribution of rankings over $[n] \setminus I$ is the same as that of $\Mal{\phi}{\pi \setminus {I}}$.
\end{lemma}
The following lemma describe a useful trick that allows us to simulate the addition of another element that is added to the start of the central ranking $\pi$, using the knowledge of $\phi$. This will be particularly useful to simplify certain degenerate cases.

\begin{lemma}
\label{app_lem:simulate-new-element}
Let $\Mal{\phi}{\pi}$ be a Mallows model over $[n]$. Given oracle access to $\Mal{\phi}{\pi}$ and a new element $e_0\notin [n]$ we can efficiently simulate an oracle access to $\Mal{\phi}{(e_0, \pi)}$.
\end{lemma}

\subsection{Proofs of Lemmas~\ref{app_lem:conditioned-on-x}, \ref{app_lem:ignore-prefix}, \ref{app_lem:simulate-new-element}}

\noindent\textbf{Observation.} All of the properties we state and prove in this appendix are based on the following important observation. Given two permutations $\pi$ and $\pi'$, denote the first element in $\pi$ as $e_1$. Then we have that
\[ \# \textrm{pairs }(e_1, \el{i})_{i\neq 1} \textrm{ that } \pi,\pi' \textrm{ disagree on } = \left(\textrm{position of }e_1 \textrm{ in } \pi'\right) -1 = pos_{\pi'}(e_1)-1\]
The same holds for the last element, denoted $e_n$, only using the distance between $pos_{\pi'}(e_n)$ and the $n$th-position (i.e., $n-pos_{\pi'}(e_n)$).

We begin by characterizing $\Znn$.
\begin{property}
For every $n$ and any $\pi_0\in S_n$ we have that $\Znn = \sum_{\pi} \phi^{d_{\sf kt}(\pi,\pi_0)} = \prod_{i=1}^n Z_i = \prod_i\left(\sum_{j=0}^{i=1}\phi^j\right)$.
\end{property}
\begin{proof}
By induction on $n$. For $n=1$ there's a single permutation over the set $\{1\}$ and $Z_1 = 1$. For any $n>1$, given a permutation over $n$ elements $\pi\in S_n$, denote its first element as $e_\pi$. Based on our observation, we have that
\[\dkt{\pi,\pi_0} = \#\textrm{swaps involving }e_\pi + \dkt{\pi\setminus e_\pi,  \pi_0 \setminus e_\pi} = (pos_{\pi_0}(e_\pi)-1)+ \dkt{\pi\setminus e_\pi, \pi_0\setminus e_\pi}\]
And so we have
\begin{eqnarray*}
& \Znn = \sum_{\pi} \phi^{d_{\sf kt}(\pi,\pi_0)} & = \sum_{j=1}^n \sum_{\{\pi: e_\pi \textrm{ is the $j$th elements in } \pi_0\}} \phi^{d_{\sf kt}(\pi,\pi_0)} \cr
&& = \sum_{j=1}^n \sum_{\{\pi: e_\pi \textrm{ is the $j$th elements in } \pi_0\}}\phi^{j-1}\phi^{d_{\sf kt}(\pi\setminus e_\pi,\pi_0 \setminus e_\pi)} \cr
&& = \sum_{j=0}^{n-1} \phi^{j}\sum_{\pi\in S_{n-1}}\phi^{d_{\sf kt}(\pi,\pi_0^{-j})} \cr
&& \stackrel {\rm induction} = \sum_{j=0}^{n-1} \phi^{j}\left(\prod_{i=1}^{n-1}Z_i\right) =  \left(\prod_{i=1}^{n-1}Z_i\right) Z_n = \prod_{i=1}^{n}Z_i  \cr
\end{eqnarray*}
where $\pi_0^{-j}$ denotes the permutation we get by omitting the $j$th element from $\pi_0$.
\end{proof}

Observe that the proof essentially shows how to generate a random ranking from a Mallows model. What we in fact showed is that the given a permutation $\pi = (e, \pi\setminus e)$ we have that
\[ \Pr[\pi] = \tfrac 1 {\Znn} \phi^{(pos_{\pi_0}(e)-1) + \dkt{\pi\setminus e, \pi_0\setminus e}} = \frac {\phi^{(pos_{\pi_0}(e)-1)}} {Z_n} \cdot \frac {\phi^{\dkt{\pi\setminus e, \pi_0\setminus e}}} {\Zpart{1:(n-1)}}\]
And so, to generate a random permutation using $\pi_0$: place the $j$th elements of $\pi_0$ at the first position w.p. $\propto \phi^{j-1}$, and recourse over the truncated permutation. $\pi_0\setminus e_1$ to find the rest of the permutation (positions $1,2,\ldots,j-1,j+1,\ldots,n$). This proves Lemma~\ref{app_lem:conditioned-on-x}. 

Note the symmetry between $\pi$ and $\pi_0$ in defining the weight of $\pi$. Therefore, denoting $e_1$ as the element $\pi_0$ ranks at the first position, we have that
\[\dkt{\pi,\pi_0} = \#\textrm{swaps involving }e_1 + \dkt{\pi\setminus e_1,  \pi_0 \setminus e_1} = (i-1)+ \dkt{\pi\setminus e_1, \pi_0\setminus e_1}\]
and so, the probability of permutation $\pi$ in which $e_1$ is ranked at position $j$ and the rest of the permutation is as a given permutation $\sigma$ over $n-1$ elements is:
\[ \Pr[\pi] = \tfrac 1 {\Znn} \phi^{(j-1) + \dkt{\pi\setminus e_1, \pi_0\setminus e_1}} = \frac {\phi^{(j-1)}} {Z_n} \cdot \frac {\phi^{\dkt{\pi\setminus e_1, \pi_0\setminus e_1}}} {\Zpart{1:(n-1)}}\]
So, an alternative way to generate a random permutation using $\pi_0$ is to rank element $e_1$ at position $j$ w.p. $\propto \phi^{j-1}$ and then to recourse over the truncated permutation $\pi_0\setminus e_1$. Repeating this argument for each element in a given prefix $I$ of $\pi_0$ proves Lemma~\ref{app_lem:ignore-prefix}. 

Observe that the algorithms the generate a permutation for a given Mallows model also allow us to simulate a random sample from a Mallows model over $n+1$ elements. That is, given $\pi_0$, we can introduce a new element $e_0$ and denote $\pi_0' = (e_0 , \pi_0)$. Now, to sample from a Mallows model centered at $\pi'_0$ all we need is to pick the position of $e_0$ (moving it to position $j$ w.p. $\phi^{j-1}/\Zpart{n+1}$), then sampling from original Mallows model. This proves Lemma~\ref{app_lem:simulate-new-element}. 

\subsection{Total Variation Distance}
\label{apx_subsec:TV}

In this subsection, our goal is to prove Lemma~\ref{lem:simulate-noisy-oracle}. Namely, we aim to show that given $\phi$, for every $\delta > 0$ we can pick any $\hat\phi$ sufficiently close to $\phi$, and have that the total variation distance between the two models $\Mal{\phi}{\pi_0}$ and $\Mal{\hat\phi}{\pi_0}$ is at most $\delta$.
\begin{proof}[Proof of Lemma~\ref{lem:simulate-noisy-oracle}]
First, denote $\phi = e^{-\beta}$ and $\hat\phi = e^{-\hat\beta}$. And so it holds that  
\[ |\beta-\hat\beta| = |\ln(1/\phi)-\ln(1/\hat\phi)| = |\ln(\hat\phi/\phi)| \leq |\ln(1+\tfrac {|\phi-\hat\phi|}{\phi_{\min}})| \leq \tfrac {|\phi-\hat\phi|}{\phi_{\min}}\] assuming some global lower bound $\phi_{\min}$ on $\phi,\hat\phi$.

Observe that for every $\pi$ we have that 
\[\phi^{\dkt{\pi,\pi_0}} = \exp(-\beta \dkt{\pi,\pi_0}) = \exp(-\hat\beta \dkt{\pi,\pi_0})\exp(-(\beta-\hat\beta) \dkt{\pi,\pi_0}) \leq e^{\tfrac 1 2 n^2 |\beta-\hat\beta|} \hat\phi^{\dkt{\pi,\pi_0}}  
\]
Summing over all permutation (and replacing the role of $\phi$ and $\hat\phi$) we have also that $\sum_\pi \phi^{\dkt{\pi,\pi_0}} \geq e^{-\tfrac 1 2 n^2 |\beta-\hat\beta|} \sum_\pi \hat\phi^{\dkt{\pi,\pi_0}}$.
Let $p_\pi$ (resp. $\hat p_\pi$) denote the probability of sampling the permutation $\pi$ from a Mallows model $\Mal{\phi}{\pi_0}$ (resp. $\Mal{\hat\phi}{\pi_0}$). It follows that for every $\pi$ we have  
\[ p_\pi = \frac {\phi^{\dkt{\pi,\pi_0}}} {\sum_{\pi'}\phi^{\dkt{\pi',\pi_0}}} \leq e^{n^2 |\beta-\hat\beta|} \frac {\hat\phi^{\dkt{\pi,\pi_0}}} {\sum_{\pi'}\hat\phi^{\dkt{\pi',\pi_0}}} = e^{n^2 |\beta-\hat\beta|} \hat p_\pi \]
and similarly, $\hat p_\pi \leq e^{n^2 |\beta-\hat\beta|} p_\pi$. 

Therefore, assuming that $|\beta-\hat\beta|$ is sufficiently small, and using the fact that $|1-e^x| \leq 2|x|$ for $x\in (-\tfrac 1 2,\tfrac 1 2)$, then we have
\begin{eqnarray*}
& \|\Mal{\phi}{\pi} - \Mal{\hat\phi}{\pi} \|_{\textrm{TV}} & =\frac 1 2 \sum_\pi |p_\pi - \hat p_\pi| \cr
&& =\frac 1 2 \sum_\pi p_\pi \left|1-\frac {\hat p_\pi}{p_\pi}\right| \leq \frac 1 2\sum_\pi 2p_\pi  n^2|\beta-\hat\beta| = \frac{n^2}{\phi_{\min}}|\phi-\hat\phi|
\end{eqnarray*}
It follows that in order to bound the total variation distance by $\delta$ we need to have $\phi$ and $\hat\phi$ close up to a factor of $\delta \cdot \phi_{\min}/n^2$.
\end{proof}

\section{Algorithm and Subroutines}\label{app:algorithm}
We now describe the algorithm and its subroutines in full detail. These will be followed by the analysis of the algorithms and proof of correctness in the following sections. 
Broadly speaking, our algorithm (Algorithm~\ref{alg:main}) has two main components.  

\begin{fragment*}[t]
\caption{\label{app_alg:main}{\sc Learn Mixtures of Two Mallows models}, \textbf{Input: } a set $\cal{S}$ of $N$ samples from $\MMix$, Accuracy parameters $\epsilon, \epsilon_2$.  \vspace*{0.01in}
}
\begin{enumerate} \itemsep 0pt
\small 
\item Set threshold $\epsilon_2 = f_2(\epsilon)$.
\item Let $\cP$ be the empirical estimate of $P$ on samples in $\cal{S}$. 
\item Run $O(\log n)$ times
\begin{enumerate}
\item Partition $[n]$ randomly into $S_a$, $S_b$ and $S_c$.
\item Set $T^{(abc)} = \big ( \cP_{ijk} \big )_{i \in S_a, j \in S_b, k \in S_c}$.
\item Run \tensors as in Theorem 4.2 of(\cite{BCV}) to get a decomposition of $T_{abc} = \ua \otimes \ub \otimes \uc + \va \otimes \vb \otimes \vc$.
\item Let $\Maa = (\ua ; \va)$, $\Mbb = (\ub ; \vb)$, $\Mcc = (\uc ; \vc)$.
\item If $\text{min} (\sigma_2(\Maa), \sigma_2(\Mbb), \sigma_2(\Mcc)) \geq \epsilon_2$,
\begin{enumerate}
\item $(\cw_1, \cw_2, \cphi_1, \cphi_2, {\pi_1}^{\prime}, {\pi_2}^{\prime}) \leftarrow $ {\sc Infer-Top-k($\cP$, $\Maa$, $\Mbb$, $\Mcc$)}.
\item $(\cpi_1, \cpi_2) \leftarrow $ {\sc Recover-Rest(${\cal{S}}, {\cw}_1, {\cw}_2, {\cphi}_1, {\cphi}_2, {{\pi_1}^{\prime}}, {{\pi_2}^{\prime}}$, $\epsilon_2/\sqrt{2n}$)}.\\
Return {\sc Success} and output $(\cw_1, \cw_2, \cphi_1, \cphi_2, \cpi_1, \cpi_2)$.
\end{enumerate}   
\item Else if $\sigma_2(\Maa) < \epsilon_2$ and $\sigma_2(\Mbb) \geq \epsilon_2$, and $\sigma_2(\Mcc) \geq \epsilon_2$~(or other symmetric cases), \\
let $\pa = \left(\cP_i\right)_{i \in S_a}$.\\ $\cphi \leftarrow $ {\sc Estimate-Phi($\pa$)}. 
\item Else $\cphi = \text{median}\left({\text{\sc Estimate-Phi}(\pa)}, {\text{\sc Estimate-Phi}(\pb)}, {\text{\sc Estimate-Phi}(\pc)}\right)$.
\item Else,  (at least two of the three matrices $\Maa,\Mbb,\Mcc$ are essentially rank-1)\\
let $\tau\in\{a,b,c\}$ denote a matrix $M'_\tau$ s.t. $\sigma_2(M'_\tau)<\eps_2$, and let $p^{(\tau)} = (\cP_i)_{i\in S_\tau}$.\\
$\hat\phi\leftarrow \textsc{Estimate-Phi}(p^{(\tau)})$.
\end{enumerate}
\item Run {\sc Handle-Degenerate-Case($\cP$, $\hat{\phi}, \epsilon$)}.
\end{enumerate}
\end{fragment*}

\paragraph{Retrieving the Top Elements and Parameters.}
In the first part we use spectral methods to recover elements which have a good chance of appearing in the first position. The algorithm tries $O(\log n)$ different random partitions $S_a,S_b,S_c$, and constructs the tensor $T^{(abc)}$ from the samples as described in step 3(b). We then try to find a rank-$2$ decomposition of the tensor using a black-box algorithm for decomposing non-symmetric tensors. While we use the algorithm of \cite{BCV} here, we can use the more practically efficient algorithm of Jennrich~\cite{BCMV}, or other power-iteration methods that are suitably modified to handle non-symmetric tensors. 

These algorithms work when the factor matrices $\Ma, \Mb, \Mc$ have polynomially bounded condition number (in other words their second largest singular values $\sigma_2(\cdot)$ is lower bounded by a polynomial in the input parameters) --- in such cases the tensor $T^{(abc)}$ has a unique rank-$2$ decomposition. If this condition holds for any of the random partitions, then one can recover the top few elements of both $\pi_1$ and $\pi_2$ correctly. In addition, we can also infer the parameters $w$'s and $\phi$'s to good accuracy $\epsilon$ ~(corresponding to {\sc Infer-Top-k} (Algorithm~\ref{alg:inferparams}). This is detailed in section~\ref{app:tensors}.

If any random partition $S_a, S_b, S_c$ fails to produce a tensor $T^{(abc)}$ with well-conditioned factor matrices, then we are already in a special case. We show that in this case, the scaling parameters $\phi_1 \approx \phi_2$ with high probability. We exploit the random choice of the partition to make this argument (see Lemma~\ref{lem:eqphi}). However, we still need to find the top few elements of the permutations and the weights. If all these $O(\log n)$ random partitions fail, then we show that we are in the {\em Degenerate case} that we handle separately; we describe a little later. Otherwise, if at least one of the random partitions succeeds, then we have estimated the scaling parameters, the mixing weights and the top few elements of both permutations.   

\paragraph{Recovering Rest of the Elements.}

The second part of the algorithm~(corresponding to {\sc Recover-Rest}) takes the inferred parameters and the initial prefixes as input and uses this information to recover the entire rankings $\pi_1$ and $\pi_2$. This is done by observing that the probability of an element $\el{i}$ going to position $j$ can be written as a weighted combination of the corresponding probabilities under $\pi_1$ and $\pi_2$. In addition, as mentioned in Section~\ref{sec:prelims}, the reduced distribution obtained by conditioning on a particular element $e_j$ going to position $1$ is again a mixture of two Mallows models with the same parameters. Hence, by conditioning on a particular element which appears in the initial learned prefix, we get a system of linear equations which can be used to infer the probability of every other element $\el{i}$ going to position $j$ in both $\pi_1$ and $\pi_2$. This will allow us to infer the entire rankings.

\begin{fragment*}[t]
\caption{\label{alg:inferparams}{\sc Infer-Top-k}, \textbf{Input: } $\cP, \Maa = (u^{(a)}; v^{(a)}), \Mbb = (u^{(b)}; v^{(b)}), \Mcc = (u^{(c)}; v^{(c)})$.  \vspace*{0.01in}
}
\begin{enumerate} \itemsep 0pt
\small 
\item Let $\hat{P}_a = \hat{P}(i \in a)$
\item Set $(\alpha_a, \beta_a)^T = \left(M'_a \right)^{\dagger} \hat{P}_a$\\
          $(\alpha_b, \beta_b)^T = \left(M'_b \right)^{\dagger} \hat{P}_b$\\
          $(\alpha_c, \beta_c)^T = \left(M'_c \right)^{\dagger} \hat{P}_c$.
\item Set $\hat{w_1} = \norm{\alpha_a u^{(a)}}_1 + \norm{\alpha_b u^{(b)}}_1 + \norm{\alpha_c u^{(c)}}_1$, $\hat{w_2} = 1-\hat{w_1}$.
\item Let $u = \left(\frac{\alpha_a}{w_1} u^{(a)}, \frac{\alpha_b}{w_1} u^{(b)}, \frac{\alpha_c}{w_1} u^{(c)} \right)$.\\
$v = \left(\frac{\beta_a}{w_2} v^{(a)}, \frac{\beta_b}{w_2} v^{(b)}, \frac{\beta_c}{w_2} v^{(c)} \right)$.
\item Sort the vectors $u$ and $v$ in decreasing order, i.e., $U \leftarrow ${\sc sort($u$)}, $V \leftarrow ${\sc sort($v$)}.
\item $\hat{\phi_1} = \frac{U_2}{U_1}$ and $\hat{\phi_2} = \frac{V_2}{V_1}$.
\item Define $\gamma = \frac{(1-\hat{\phimax})^2}{4n\hat{\phimax}}$. Let $r_1 = \log_{1/\hat{\phi_1}} \left(\frac{n^{10}}{\wmin^2 \gamma^2} \right)$ and $r_2 = \log_{1/\hat{\phi_2}} \left(\frac{n^{10}}{\wmin^2 \gamma^2} \right)$.
\item Output $\pi^{\prime}_1$ to be the first $r_1$ ordered elements according to $U$ and $\pi^{\prime}_2$ to be the first $r_2$ ordered elements according to $V$.
\end{enumerate}
\end{fragment*}

\paragraph{Degenerate Cases.}
In the case when none of the random partition produces a tensor which has well-conditioned factor matrices (or alternately, a unique rank-$2$ decomposition), the instance is a very special instance, that we term {\em degenerate}. The additional subroutine~({\sc Handle-Degenerate-Case}) takes care of such degenerate instances. Before we do so, we introduce some notation to describe these degenerate cases. 

\noindent {\em Notation.} Define $L_{\epsilon} = \{\el{i}: P_i \ge \epsilon\}$. If $\epsilon$ not stated explicitly $L$ refers to $L_{\sqrt{\epsilon}}$ where $\epsilon$ is the accuracy required in Theorem~\ref{thm:main}. 

Now we have the following definition that helps us formally define the degenerate case.
\begin{definition} [Bucketing by relative positions]
For every $\ell \in \bbZ$, let $B_\ell = \set{\el{i} \in L: pos_{\pi_1}(\el{i}) - pos_{\pi_2}(\el{i})=\ell}$. Further let $\ell^*$ be the majority bucket for the elements in $L$. 
\end{definition}

We call a mixture model $\MMix$ as degenerate if except for at most $2$ elements, all the elements in $L$ fall into the majority bucket. In other words, $|\ell^*| \ge |L| - 2$. Intuitively, in this case one of the partitions $S_a, S_b, S_c$ constructed by the algorithm will have their corresponding $u$ and $v$ vectors as parallel to each other and hence the tensor method will fail. We show that when this happens, it can be detected and in fact this case provides useful information about the model parameters. More specifically, we show that in a degenerate case, $\phi_1$ will be almost equal to $\phi_2$ and the two rankings will be aligned in a couple of very special configurations~(see Section~\ref{app:degenerate_cases}). Procedure {\sc Handle-Degenerate-Case} is designed to recover the rankings in such scenarios. 

\section{Retrieving the Top elements}\label{app:tensors}

Here we show how the first stage of the algorithm i.e. {\em steps (a)-(e.i)} manages to recover the top few elements of both rankings $\pi_1$ and $\pi_2$ and also estimate the parameters $\phi_1,\phi_2, w_1, w_2$ up to accuracy $\epsilon$. We first show that if $\Ma, \Mb, \Mc$ have non-negligible minimum singular values (at least $\epsilon'_2$ as in Lemma~\ref{lem:tensoralg}), then the decomposition is unique, and hence we can recover the top few elements and parameters from {\sc Infer Top-K}. Otherwise, we show that if this procedure did not work for all $O(\log n)$ iterations, we are in \emph{the degenerate case} (Lemma~\ref{lem:eqphi} and Lemma~\ref{lem:bucket2}), and handle this separately.

For the sake of analysis, we denote by $\minl$ the smallest length of the vectors in the partition i.e. $\minl=\min_{\tau \in \set{a,b,c}} \min\set{\norm{x^{(\tau)}},\norm{y^{(\tau)}}}$. Lemma~\ref{lem:probspread} shows that with high probability $\minl \ge \phimin^{C \log n} (1-\phi)$ for some large constant $C$. 
 
The following lemma shows that when $\Ma, \Mb, \Mc$ are well-conditioned, Algorithm {\sc TensorDecomp} finds a decomposition close to the true decomposition up to scaling. This Lemma essentially follows from the guarantees of the Tensor Decomposition algorithm in ~\cite{BCV}. It also lets us conclude that $\sigma_2(\Maa), \sigma_2(\Mbb), \sigma_2(\Mcc)$ are all also large enough. Hence, these singular values of the matrices $\Maa, \Mbb, \Mcc$ that we obtain from \tensors algorithm can be tested to check if this step worked. 
\begin{lemma}[Decomposition guarantees]\label{lem:tensoralg}
In the conditions of Theorem~\ref{thm:main}, suppose there exists a partition $S_a, S_b, S_c$ such that the matrices $\Ma=(\xa ; \ya),\Mb=(\xb ; \yb)$ and $\Mc=(\xc ; \yc)$ are well-conditioned i.e. $\sigma_2 (\Ma), \sigma_2(\Mb), \sigma_2(\Mc) \ge \epsb$, then with high probability, Algorithm {\sc TensorDecomp} finds $\Maa=(\ua; \va), \Mbb=(\ub; \vb), \Mcc=(\uc; \vc)$  such that
\begin{enumerate}
\item For $\tau \in \set{a, b, c}$, we have $u^{(\tau)} = \alpha_a x^{(\tau)}+ z_1^{(\tau)}$ and $v^{(\tau)} = \beta_a y^{(\tau)}+ z_2^{(\tau)}$ where $\norm{z_1^{(\tau)}},\norm{z_2^{(\tau)}} \le \errp_{\ref{lem:tensoralg}}(n,\eps,\eps_2,\wmin)$ 
\item $\sigma_2(\Maa) \ge \minl(\epsb- \errp_{\ref{lem:tensoralg}})$ (similarly for $\Mbb, \Mcc$).
\end{enumerate}
where $\errp_{\ref{lem:tensoralg}}$ is a polynomial function $\errp_{\ref{lem:tensoralg}}= \min\set{\sqrt{\errp_{\text{tensors}}(n,1,\kappa=\frac{1}{\eps_2},\eps_s n^{3/2})},\frac{\minl^4 \wmin}{4}}$ and $\errp_{\text{tensors}}$ is the error bound attained in Theorem 2.6 of \cite{BCV}.
\end{lemma}
\begin{proof}
Let $\epsta=\errp_{\ref{lem:tensoralg}}$.  The entry-wise sampling error is $\eps_s \le 3\log n/\sqrt{N}$. Hence, the rank-$2$ decomposition for $T^{(abc)}$ is $n^{3/2}\eps_s$ close in Frobenius norm. We use the algorithm given in ~\cite{BCV} to find a rank-$2$ decomposition of $T^{(abc)}$ that is $O(\eps_s)$ close in Frobenius norm. Further, the rank-$1$ term $\ua \otimes \ub \otimes \uc$ is $\epsta^2$-close to $w_1 c_3(\phi_1) \xa \otimes \xb \otimes \xc$.  Let us renormalize so that $\norm{\ua}=\norm{\ub}=\norm{\uc} \ge \wmin^{1/3} \minl$.  

Applying Lemma~\ref{lem:error1}, we see that $\ua = \alpha_a \xa+ z^{(a)}_1$ where $\norm{z^{(a)}_1}\le \epsta$, and similarly $\va=\beta_a \ya+z^{(a)}_2$ where $\norm{z_2} \le \epsta$. Further $\wmin^{1/3} \minl \phi_1 /4 \le \alpha_a \le 1/\minl $. Further
$$\sigma_2 \left( \alpha_a \xa ; \beta_a \ya \right)\ge \min\set{\alpha_a, \beta_a} \sigma_2(\Ma) \ge \frac{\wmin^{1/3} \minl \phi_1}{4} \sigma_2(\Ma).$$

Hence, $\sigma_2(\Maa) \ge \wmin^{1/3}\minl \phi_1 \sigma_2(\Ma)/2 - 2 \epsta$, as required. The same proof also works for $\Mbb, \Mcc$. 
\end{proof}

Instead of using the enumeration algorithm of \cite{BCV}, the simultaneous eigen-decomposition algorithms in \cite{BCMV} and \cite{GVX} can also be used. The only difference is that the ``full-rank conditions'' involving the $\Ma,\Mb,\Mc$ are checked in advance, using the empirical second moment. 
Note that \tensors only relies on elements that have a non-negligible chance of appearing in the first position $L$: this can lead to large speedup for constant $\phi_1, \phi_2 < 1$ by restricting to a much smaller tensor. 

Lemma~\ref{lem:fullrank} captures how Algorithm~\ref{alg:main} (steps 3 (a - e.i)) performs the first stage using Algorithm~\ref{alg:inferparams} and recovers the weights $w_1,w_2$ and $x,y$ when the factor matrices $\Ma,\Mb,\Mc$ are well-conditioned.    

In the proof we show that in this case, for one of the $O(\log n)$ random partitions, Lemma~\ref{lem:tensoralg} succeeds and recovers vectors $\ua, \va$ which are essentially parallel to $\xa$ and $\ya$ respectively (similarly for $\ub,\uc,\vb,\vc$). Sorting the entries of $\ua$ would give the relative ordering among those in $S_a$ of the top few elements of $\pi_1$. However, to figure out all the top-$k$ elements, we need to figure out the correct scaling of $\ua, \ub, \uc$ to obtain $\xa$. This is done by setting up a linear system.

Now we present the complete proof of the lemmas.

\subsection{Proof of Lemma~\ref{lem:fullrank}: the Full Rank Case}
If such a partition $S^*_a, S^*_b, S^*_c$ exists such that $\sigma_2(\Ma)\ge \epsb$, then there exists a 2-by-2 submatrix of $\Ma$ corresponding to elements $\el{i_1}, \el{j_1}$ which has $\sigma_2(\cdot )\ge \epsb$. Similarly there exists such pairs of elements $\el{i_2}, \el{j_2}$ and $\el{i_3}, \el{j_3}$ in $S_b$ and $S_c$ respectively. But with constant probability the random partition $S_a, S_b, S_c$ has $\el{i_1}, \el{j_1} \in S_a$,   $\el{i_2}, \el{j_2} \in S_b$, $\el{i_3}, \el{j_3} \in S_c$ respectively. Hence in the $O(\log n)$ iterations, at least one iteration will produce sets $S_a, S_b, S_c$ such that $\sigma_2(\Ma), \sigma_2(\Mb), \sigma_2(\Mc) \ge \epsb$ with high probability.  Further, Lemma~\ref{lem:tensoralg} also ensures that $\sigma_2(\Maa), \sigma_2(\Mbb), \sigma_2(\Mcc) \ge \eps_2$.

Lemma~\ref{lem:tensoralg} recovers vectors $\ua, \va$ which are essentially parallel to $\xa$ and $\ya$ respectively (similarly for $\ub,\uc,\vb,\vc$). While sorting the entries of $\ua$ would give the relative ordering among those in $S_a$ of the top few elements of $\pi_1$, we need to figure out the correct scaling of $\ua, \ub, \uc$ to recover the  top few elements of $\pi_1$.  

\newcommand{\za}{z^{(a)}}
\newcommand{\zb}{z^{(b)}}
\newcommand{\zc}{z^{(c)}}

From Lemma~\ref{lem:tensoralg}, we can express 
$$w_1 \xa = \alpha'_a \ua+ \za_1 ~\text{where}~ \za_1 \perp \ua 
\text{where} ~\norm{\za_1} \le \errp_{\ref{lem:tensoralg}}(n,\eps_s,\epsb).$$ 
 Similarly $w_2 \ya=\beta'_a \va +\za_2$, where $\norm{\za_2} \le \errp_{\ref{lem:tensoralg}}$. If $\eps_s$ is the sampling error for each entry in $\pa$, we have
\begin{align}
\norm{ w_1 \xa + w_2 \xb - \pa} &< \sqrt{n} \eps_s\\
 \norm{\alpha'_a \ua + \beta \va - \pa } &< \sqrt{n} \eps_s + \frac{1}{2}\wmin^{1/3} \phi_1 \minl \errp_{\ref{lem:tensoralg}} \label{eq:scaling}
 \end{align} 
Eq~\eqref{eq:scaling} allows us to define a set of linear equations with unknowns $\alpha'_a, \beta'_a$, constraint matrix given by  $\Maa=(\ua ; \va)$. Hence, the error in the values of $\alpha'_a, \beta'_a$ is bounded by the condition number of the system and the error in the values i.e. $$\eps_{\alpha}\le \kappa(\Maa).\wmin^{1/3} \minl \errp_{\ref{lem:tensoralg}} \le \left(\frac{1}{4}\wmin^{1/3}\phimin \minl \epsb-\errp_{\ref{lem:tensoralg}} \right)^{-1}\cdot\frac{\phimin}{2}\wmin^{1/3}\minl \errp_{\ref{lem:tensoralg}}.$$  The same holds for $\alpha_b, \alpha_c, \beta_b, \beta_c$.

However, we also know that $\norm{\xa}_1 + \norm{\xb}_1 + \norm{\xc}_1 =1$. Hence, 
$$
\abs{\norm{\alpha_a \ua}_1 + \norm{\alpha_b \ub}_1 + \norm{\alpha_c \uc}_1 -w_1} \le \eps \le 3\sqrt{n}(\eps_{\alpha}+\errp_{\ref{lem:tensoralg}}).$$ 

Thus, $\cw_1, \cw_2$ are within $\eps$ of $w_1, w_2$.  Hence, we can recover vectors $x$ by concatenating $\frac{\alpha_a}{w_1}\ua, \frac{\alpha_b}{w_1}\ub, \frac{\alpha_c}{w_1} \uc$ (similarly $y$). 
Since we have $\errp_{\ref{lem:tensoralg}}< \phi_1(1-\phi)/\wmin$, it is easy to verify that by sorting the entries and taking the ratio of the top two entries, $\cphi_1$ estimates $\phi_1$ up to error $\frac{2\errp_{\ref{lem:tensoralg}} \phi_1(1-\phi_1)}{\wmin}$ (similarly $\phi_2$).
Finally, since we recovered $x$ up to error $\eps''= \frac{2\errp_{\ref{lem:tensoralg}}}{\wmin}$, we recovered the top $m$ elements of $\pi_1$ where $m \le \log_{\phi_1}\left(2 \errp_{\ref{lem:tensoralg}} (1-\phi_1)/\wmin \right)$.

\section{Degenerate Case}\label{app:degenerate_cases}

\begin{fragment*}[t]
\caption{\label{alg:commonprefix}{\sc Remove-Common-Prefix}, \textbf{Input: } a set $\cal{S}$ of $N$ samples from $\MMixSame$, $\epsilon$.  \vspace*{0.01in}
}
\begin{enumerate} \itemsep 0pt
\small 
\item Initialize $I \leftarrow \emptyset$, $S = [n]$. 
\item for $t = 1$ to $n$,
\begin{enumerate}
\item For each element $x \in [n] \setminus I$, estimate $\hat{p}_{x,1} = Pr(x \text{ goes to position } t)$.
\item Let $x_t = \arg\max_{x \in [n] \setminus I} \hat{p}_{x,1}$.
\item If $|\hat{p}_{x,1} - \frac{1}{Z_{n-t+1}}| > \errp(\epsilon)$, return $I$ and {\sc Quit}.
\item Else $I \leftarrow I \cup {x_t}$
\end{enumerate}
\item Output $I$.
\end{enumerate}
\end{fragment*}

While we know that we succeed when $\Ma,\Mb,\Mc$ have non-negligible minimum singular value for one of the the $O(\log n)$ random partitions, we will now understand when this does not happen. 

Recollect that $L=L_{\sqrt{\eps}}=\set{\el{i}: P_i \ge \sqrt{\epsilon}}$. For every $\ell \in \bbZ$, let $B_\ell = \set{\el{i} \in L: \pi_1^{-1}(i) - \pi_2^{-1}(i)=\ell}$. Further let $\ell^*$ be the majority bucket for the elements in $L$. 
We call a mixture model $\MMix$ as degenerate if the parameters of the two Mallows models are equal, and
except for at most $2$ elements, all the elements in $L$ fall into the majority bucket. In other words, $|\ell^*| \ge |L| - 2$.

We first show that if the tensor method fails, then the parameters of the two models $\phi_1$ and $\phi_2$ are essentially the same. Further, we show how the algorithm finds this parameter as well.

\begin{lemma}[Equal parameters]\label{lem:eqphi}
In the notation of the Algorithm~\ref{alg:main}, for any $\eps'>0$, suppose $\sigma_2(\Maa) < \eps_2 \le \errp_{\ref{lem:eqphi}}(n,\eps',\wmin,\phi_1,\phi_2)$ (or $\Mbb, \Mcc$), then with high probability ($1-1/n^3$), we have that $\abs{\phi_1 - \phi_2 } \le \eps'$ and further Algorithm~\ref{alg:estimatephi} ({\sc Estimate-Phi}) finds $\abs{\cphi-\phi_1} \le \eps'/2$. The number of samples needed $N>\text{poly}(n,\frac{1}{\eps'})$.
\end{lemma}
This lemma is proven algorithmically. We first show that Algorithm~\ref{alg:estimatephi} finds a good estimate $\cphi$ of $\phi_1$. However, by the same argument $\cphi$ will also be a good estimate of $\phi_2$! Since $\cphi$ will be $\eps'/2$-close to both $\phi_1$ and $\phi_2$, this will imply that $\abs{\phi_1-\phi_2} \le \eps'$ ! We prove this formally in the next section. But first, we first characterize when the tensor $T^{(abc)}$ does not have a unique decomposition --- this characterization of uniqueness of rank-$2$ tensors will be crucial in establishing that $\phi_1 \approx \phi_2$. 
 
\subsection{Characterizing the Rank and Uniqueness of tensor $T^{(abc)}$ based on $\Ma, \Mb,\Mc$}
\label{app:characterization}

To establish Lemma~\ref{lem:eqphi}, we need the following simple lemma, which establishes that the conditioning of the matrices output by the Algorithm TensorDecomp is related to the conditioning of the parameter matrices $\Ma, \Mb, \Mc$.

\begin{lemma}[Rank-$2$ components]\label{lem:rank2}
Suppose we have sets of vectors $\left(g_i, h_i, g'_i, h'_i\right)_{i = 1 , 2 ,3}$ with length at most one ($\norm{\cdot}_2 \le 1$) such that
$$T= g_1 \otimes g_2 \otimes g_3 + h_1 \otimes h_2 \otimes h_3  \text{ and } \norm{T-g'_1 \otimes g'_2 \otimes g'_3 + h'_1 \otimes h'_2 \otimes h'_3} \le \eps_s $$  
such that matrices have minimum singular value $\sigma_2\big(g_1 ; h_1\big), \sigma_2\big(g_2; h_2\big) \ge \lambda$ and $\norm{g_3}, \norm{h_3} \ge \minl$, 
then we have that for matrices $M'_1 = \big(g'_1 ; h'_1\big), M'_2=\big(g'_2; h'_2\big)$
$$\sigma_2(M'_1) \ge \frac{\lambda^2 \minl}{4n} - \eps_s \text{ and }  \sigma_2(M'_1) \ge \frac{\lambda^2 \minl}{4n} - \eps_s.$$
\end{lemma}
\begin{proof}
Let matrices $M_1 = \big(g_1 ; h_1\big), M_2=\big(g_2; h_2\big)$.
For a unit vector $w$ (of appropriate dimension) let 
\begin{align*}
M_w = T(\cdot, \cdot, w)&= \iprod{w}{g_3} g_1 \otimes g_2 + \iprod{w}{h_3} h_1 \otimes h_2\\
= M_1 D_w M_2^T & \text{where } D_w = \left( \begin{matrix} \iprod{w}{g_3} & 0\\ 0 & \iprod{w}{h_3}  \end{matrix} \right). 
\end{align*}
Besides, since $w$ is a random gaussian unit vector,  $\Pr{\abs{\iprod{w}{g_3}} \ge \norm{g_3}/4\sqrt{n} }$ with probability $> 1/2$. Hence, using there exists a unit vector $w$ such that
$\min\{\abs{\iprod{w}{g_3}} , \abs{\iprod{w}{h_3}}\} \ge \minl/(4 \sqrt{n}) $.  
Hence, $$\sigma_2(M_w) \ge \frac{\lambda^2 \minl}{4 \sqrt{n}}.$$

$$\text{However,} ~\norm{M_w - M'_1 D'_w (M'_2)^T}_F \le \eps_s ~\text{where}~ D'_w = \left( \begin{matrix} \iprod{w}{g'_3} & 0 \\ 0 & \iprod{w}{h'_3} \end{matrix}\right).$$
Hence, $\sigma_2\left(M'_1 D'_w (M'_2)^T\right) \ge \sigma_2(M_w)-\eps_s$. \\
Combining this with the fact that $\sigma_2\left(M'_1 D'_w (M'_2)^T\right) \le \sigma_2(M'_1) \sigma_1(D'_w) \sigma_1(M'_2)$ gives us the claimed bound.
\end{proof}

This immediately implies the following lemma in the contrapositive. 
\begin{lemma}[Rank-1 components]\label{lem:rank1}
Suppose $\sigma_2(\Maa) < \eps$ and $\sigma_2(\Mbb)< \eps$, then two of the matrices $\Ma, \Mb, Mc$ have $\sigma_2(\cdot) < \sqrt{\frac{8\eps n}{\minl}}$, when the number of samples $N> \text{poly}(n,1/\eps)$.
\end{lemma}

\subsection{Equal Scaling Parameters}
The following simple properties of our random partition will be crucial for our algorithm. 
\begin{lemma}\label{lem:partition}
The random partition of $[m]$ into $A,B,C$ satisfies with high probability (at least $1-exp\left(-\frac{1}{C_{\ref{lem:partition}}}\cdot m\right)$):
\begin{enumerate}
\item $|A|, |B|, |C| \ge m/6$
\item There are many consecutive numbers in each of the three sets $A,B,C$ i.e.
$$ \lvert \{ i \in A ~\text{and}~ i+1 \in A \}  \rvert \ge m/100.$$
\end{enumerate}
\end{lemma}
\begin{proof}
The claimed bounds follow by a simple application of Chernoff Bounds, since each element is chosen in $A$ with probability $1/3$ independently at random. 
The second part follows by considering the $m/2$ disjoint consecutive pairs of elements, and observing that each pair fall entirely into $A$ with probability $1/9$. 
\end{proof}

\begin{lemma}\label{lem:ratiophi}
Consider a set of indices $S \subseteq [n]$ and let $p_S$ be the true probability vector $p$ of a single Mallows model $\calM(\pi,\phi)$ restricted to subset $S$. Suppose
the empirical vector $\norm{\cp_S - p_S}_{\infty} < \eps_1$,  
and there exists consecutive elements of $\pi$ in $S$ i.e. $\exists i$ such that $\pi(i),\pi(i+1) \in S$, with $p(\pi(i+1)) \ge \sqrt{\eps_1}$.
Then, if we arrange the entries of $p_S$ in decreasing order as $r_1, r_2, \dots, r_{|S|}$ we have that 
 $$\cphi = \max_{i: r_{i+1} \ge \sqrt{\eps_1} } \frac{r_{i+1}}{r_{i}}  ~\text{satisfies}~ \abs{\cphi-\phi} < 2\sqrt{\eps_1}. $$
\end{lemma}
\begin{proof}
By the properties of the Mallows  model, the ratio of any two probabilities is a power of $\phi$ i.e. $\frac{p_{\ell_2}}{p_{\ell_1}} =\phi^{\pi^{-1}(\ell_2)- \pi^{-1}(\ell_1) }$.
If $p(\pi(i+1)) \ge \sqrt{\eps_1}$, we have that
\begin{align*}
\frac{\cp(\pi(i+1))}{\cp(\pi(i))} &\le \frac{\phi\cdot p(\pi(i))+\eps_1}{p(\pi(i))-\eps_1} \\
&\le  \phi + \frac{\phi \left(p(\pi(i))-\cp(\pi(i))+\eps_1\right)}{\cp(\phi(i))} \le \phi + \eps_1 \frac{(1+\phi)}{\cp(\phi(i))}
&\le \phi + 2\sqrt{\eps_1}
\end{align*}
The same proof holds for the lower bound. 
\end{proof}

We now proceed to showing that the scaling parameters are equal algorithmically.
\begin{proof}[Proof of Lemma~\ref{lem:eqphi}]
We now proceed to prove that $\phi_1 \approx \phi_2$.  We note that $\norm{T^{(abc)}}_F \le 1$ since the entries of $T^{(abc)}$ correspond to probabilities, and for any vector $z$, $\norm{z}_2 \le \norm{z}_1$. This implies that all the vectors in the decomposition can be assumed to have $\ell_2$ norm at most $1$, without loss of generality. 
We can first conclude that at least one of the three matrices $\Ma, \Mb, \Mc$ has $\sigma_2(\cdot) <   \sqrt{\frac{8n \eps_2}{\llmin}}$.  Otherwise, we get a contradiction by applying Lemma~\ref{lem:rank2} (contrapositive) to $\Maa, \Mbb$ and $\Maa, \Mcc$. Now, we will show how the algorithm gives an accurate estimate $\cphi$ of $\phi_1$. However the exact argument applied to $\phi_2$ will show that $\cphi$ is also a good estimate for $\phi_2$, implying that $\phi_1 \approx \phi_2$.

We have two cases depending on whether one of $\sigma_2(\Mbb)$ and $\sigma_2(\Mcc)$ are non-negligible or not.

\noindent {\bf Case 1: $\sigma_2(\Mbb) \ge (\eps_2^{1/4} \frac{8n}{\llmin})^{3/4})$ and $\sigma_2(\Mcc) \ge (\eps_2^{1/4} \frac{8n}{\llmin})^{3/4})$:}  \\
Applying Lemma~\ref{lem:rank2}, we conclude that $\sigma_2(\Mb) \ge \eps_2^{1/2} (8n/\llmin)^{1/2}$ and $\sigma_2(\Mc) \ge \eps_2^{1/2} (\frac{8n}{\llmin})^{1/2}$.  However one of the matrices $\Ma, \Mb, \Mc$ has small $\sigma_2$ value. Hence $$\sigma_2(\Ma) < \eps_2^{1/2} \left(\frac{8n }{ \llmin}\right)^{1/2}= \eps'_2 ~\text{(say)}.$$

Let $\ya=\alpha \xa+\yperp$ where $\yperp \perp \xa$. Then  $\norm{\yperp} \le \eps'_2$ and $\alpha \ge \frac{\left(\norm{\ya}-\eps'_2\right)}{\norm{\xa}} \ge \minl/2$. \\ Further, $\pa=(w_1+w_2 \alpha)\xa+w_2 \yperp$. Hence,
$$ \xa=\beta \pa -  w_2\beta \yperp , ~\text{ where } 0\le \beta < \frac{2}{\minl} .$$
Since the sampling error is $\eps_s$, we have 
\begin{align*}
\xa&= \beta \cp^{(a)}+ \beta(\pa-\cp^{(a)}) -w_2 \beta \yperp\\
&= \beta \cp^{(a)}+z ~\text{where}~  \norm{z}_{\infty} \le \beta (\eps_s + \eps'_2) \le \frac{4\eps_2}{\minl}=\eps_3
\end{align*}

Consider the first  $m=C_{\ref{lem:partition}}\log n$ elements $F$ of $\pi_1$. 
\begin{align*}
\forall i \in F, x_{i} \ge \frac{\phi_1^{C_{\ref{lem:partition}} \log n}}{1-\phi_1} &\ge  \frac{n^{C_{\ref{lem:partition}}\log(1/\phi_1)}}{1-\phi_1}\\
& \ge \sqrt{\eps_3} ~\text{due to our choice of error parameters}
\end{align*}

Applying Lemma~\ref{lem:partition}, $\Omega(\log n)$ consecutive elements of $\pi_1$ occur in $S_a$. Hence applying Lemma~\ref{lem:ratiophi}, we see that the estimate $\cphi$ output by the algorithm satisfies $\abs{\cphi-\phi_1} \le 2\sqrt{\eps_3}=\frac{16\eps_2^{1/4} n^{1/4}}{\minl^{3/4}\wmin^{1/4}}$, as required. 

\noindent {\bf Case 2: $\sigma_2(\Mbb)  < (\eps_2^{1/4} \frac{8n}{\llmin})^{3/4})$:}\\
We also know that $\sigma_2(\Maa) < \eps_2$. Applying Lemma~\ref{lem:rank1}, we see that two of the three matrices $\Ma, \Mb, \Mc$ have $\sigma_2(\cdot)$ being negligible i.e.  $$\sigma_2(\cdot)< \eps_2^{1/4} \left(\frac{8n}{\llmin}\right)^{7/8}.$$
Using the same argument as in the previous case, we see that the estimates given by two of the three partitions $S_a, S_b, S_c$ is $2\sqrt{\eps_3}$ close to the $\phi_1$.  Hence the median value $\cphi$ of these estimates is also as close. 

As stated before, applying the same argument for $\phi_2$ (and $\pi_2$) , we see that $\cphi$ is $2\sqrt{\eps_3}$ close to $\phi_2$ as well. Hence, $\phi_1$ is $4 \sqrt{\eps_3}$ close to $\phi_1$. 

\end{proof}

%
\subsection{Establishing Degeneracy}

Next, we establish that if none of the $O(\log n)$ rounds were successful, then the two central permutations (restricted to the top $O(\log_{1/\phimin} n)$ positions) are essentially the same shifted by at most a couple of elements. 

\begin{lemma}\label{lem:bucket2}
Consider the large elements $L_{\sqrt{\eps}}$. Suppose $\abs{B_{\ell*}} \ge \abs{L_{\sqrt{\eps}}}-3$, then the one of the $O(\log n)$ rounds of the Tensor Algorithm succeeds with high probability.
\end{lemma}
\begin{proof}

We have two cases depending on whether  $\ell^* \le \log(\eps)/\log(\phi)$ or not.

\noindent Suppose $\abs{\ell^*} \le \log(\eps (1-\phi))/\log(\phi)$.
Let $i,j,k$ be the indices of elements in $L_{\sqrt{\eps}}$ that are not in $B_{\ell^*}$. With constant probability the random partition $S_a, S_b, S_c$ puts these three elements in different partitions. In that case, by applying Lemma~\ref{lem:bucket1} we see that $\sigma_2(\Ma), \sigma_2(\Mb), \sigma_2(\Mc) \ge \eps^2 (1-\phi)^2$. Hence, Lemma~\ref{lem:fullrank} would have succeeded with high probability. 

\noindent Suppose $\abs{\ell^*} > \log(\eps(1-\phi))/\log(\phi)$.
Assume without loss of generality that $\ell*\ge 0$. 
Consider the first three elements of $\pi_2$. They can not belong to $B_{\ell*}$ since $\ell*>3$. Hence, by pairing each of these elements with some three elements of $B_{\ell*}$, and repeating the previous argument we get that $\sigma_2(\Ma), \sigma_2(\Mb), \sigma_2(\Mc) \ge \eps^2 (1-\phi)^2$ in one of the iterations w.h.p. Hence, Lemma~\ref{lem:fullrank} would have succeeded with high probability.
\end{proof}

\begin{fragment*}[t]
\caption{\label{alg:degenerate}{\sc Handle-Degenerate}, \textbf{Input: } a set $\cal{S}$ of $N$ samples from $\MMix$, $\cphi$.  \vspace*{0.01in}
}
\begin{enumerate} \itemsep 0pt
\small 
\item $\pi^{pfx}$ $\leftarrow$ {\sc Remove-Common-Prefix($\cal{S}$)}. Let $\pi_1^{rem}=\pi_1 \setminus \pi^{pfx}, \pi_2^{rem}=\pi_2 \setminus \pi^{pfx}$.
\item If $\abs{\pi^{pfx}}= n$, then output {\sc Identical Mallows Models} and parameters $\cphi$ and $\pi^{pfx}$. 
\item Let $\calM'$ be the Mallows mixture obtained by adding three artificial elements $\el{1}^*,\el{2}^*, \el{3}^*$ to the front.
\item Run steps (1-3) of Algorithm~\ref{alg:main} on $\calM'$. If {\sc Success}, output $\cw_1, \cw_2, \pi_1, \pi_2, \cphi$.
\item If {\sc Fail}, let $\cP(i), \cP(i,j)$ be the estimates of $P(i), P(i,j)$ when samples according to $\calM'$.
\item Divide elements in $L_{\sqrt{\eps}}$ into $R \le \frac{\log(\eps \Zn(\cphi))}{2\log(\cphi)}$ disjoint sets $$ I_{r}=\set{i: \cP(i) \in \left[\frac{\cphi^{r}}{\Zn(\cphi)} -\eps, \frac{\cphi^r}{\Zn(\cphi)} +\eps\right]}.$$       
\item If $|I_r|=1$ set $\pi_1^{rem}(i)$ to be the only element in $I_r$.
\item Let $I_{bad}$ be the remaining elements in the sets $I_1 \cup I_2 \dots I_R$ along with $L_{\sqrt{\eps}} \setminus \bigcup_r I_r$. If $|I_{bad}|>4$ or $|I_{bad}|<2$, output {\sc Fail}. 
\item Let $S_a, S_b$ is any partition of $I_1 \cup I_2 \cup I_R \setminus I_{bad}$. \\
Find $i_1,j_1 \in I_{bad}$ such that $M=\left(\cP_{ij}\right)_{i \in S_a \cup \set{i_1},S_b \cup \set{j_1}}$ has $\sigma_2(M) \ge \sqrt{\eps}n$. 
\item For $i \in I_{bad}\setminus \set{i_1,j_1}$ and $i \in I_r$, set $\pi^{rem}_1(r)=\pi^{rem}_2(r)=i$.\\ 
Set $\pi_1^{rem}(1)=i_1, \pi_2^{rem}(1)=j_1$, and $\pi_1^{rem}(k)=j_1, \pi_2^{rem}(k)=i_1$ where $k \le R$ is unfilled position. 
\item Output $\pi_1=\pi^{pfx} \circ \pi_1^{rem},\pi_2=\pi^{pfx} \circ \pi_2^{pfx}, \cphi$.\\ 
Output $\cw_1, \cw_2=1-\cw_1$, by solving for $\cw_1$ from $\cP(i)= \cw_1 \pi_1^{-1}(i)+ (1-\cw_1)\pi_2^{-1}(i)$.   
\end{enumerate}
\end{fragment*}

Hence we now have two kinds of degenerate cases to deal with. The next two lemmas show how such cases are handled.

\input{degenerate}

\subsection{Auxiliary Lemmas for Degenerate Case}

\begin{lemma} \label{lem:probspread}
For any Mallows model with parameters $\phi_1,\phi_2$ has 
$$\minl=\min_{\tau \in \set{a,b,c}} \min\set{\norm{x^{(\tau)}},\norm{y^{(\tau)}}} \ge \min\set{\phi_1^{2C \log n}(1-\phi_1),\phi_2^{2C \log n}(1-\phi_2)} \text{ with probability } 1- n^{C}$$
\end{lemma}
\begin{proof}
Consider a partition $A$, and the top $m \ge 2C \log n$ elements according to $\pi$. The probability that none of the them belong to $A$ is at most $1/n^{C}$. This easily gives the required conclusion.
\end{proof}

\begin{lemma}\label{lem:bucket1}
When $\phi_1=\phi_2=\phi$, if two large elements $\el{i}, \el{j} \in L_{\sqrt{\eps}}$ belonging to different buckets $B_{\ell_1}$ and $B_{\ell_2}$ respectively with $\max\set{\abs{\ell_1}, \abs{\ell_2}} \le \frac{\log(\eps)}{ \log(\phi)}$. Suppose further that these elements are in the partition $S_a$. Then the corresponding matrix $\Ma$ has $\sigma_2(\Ma) \ge \eps^2 (1-\phi)$ when $\phi_1=\phi_2=\phi$.
\end{lemma}
\begin{proof}
Consider the submatrix $$M= \left(\begin{matrix} x_i & y_i  \\ x_j & y_j \end{matrix}\right) = x_i \left( \begin{matrix} 1 & \phi^{\ell_1}  \\ \phi^{\pi_1^{-1}(i)-\pi_1^{-1}(j)} & \phi^{\pi_1^{-1}(i)-\pi_1^{-1}(j)}\cdot \phi^{\ell_2} \end{matrix} \right) .$$ 
Using a simple determinant bound, it is easy to see that 
$$\sigma_1(M) \sigma_2(M) \ge \max\set{x_i,y_i} \max\set{ x_j, y_j}\cdot (\phi^{\abs{\ell_1}}-\phi^{\abs{\ell_2}}) \ge \max\set{x_i,y_i} \cdot \eps \phi^{\min\set{\abs{\ell_1},\abs{\ell_2}}} (1-\phi).$$
 Since $\sigma_1(M) \le 4 \max{x_i, y_i}$, we see that $\sigma_2(M) \ge \frac{\eps^2 (1-\phi)}{4}$. 
\end{proof}


\section{Recovering the complete rankings} \label{app:recover-rest}
\begin{fragment*}[t]
\caption{\label{app_alg:recover-rest}{\sc Recover-Rest}, \textbf{Input: } a set $\cal{S}$ of $N$ samples from $\MMix$, $\hat{w_1}, \hat{w_2}, \hat{\phi_1}, \hat{\phi_2}, \hat{\pi_1}, \hat{\pi_2}, \epsilon$.  \vspace*{0.01in}
}
\begin{enumerate} \itemsep 0pt
\small
\item Let $|\hat{\pi_{1}}|= r_1$, $|\hat{\pi_{2}}| = r_2$ and let $r_1 \ge r_2$ w.l.o.g. (the other case is the symmetric analog).
\item For any element $\el{i}$, define $ \hfone{i}{1}= \frac{\hat{\phi_1}^{\left(\hat{\pi_1}^{-1}(\el{i})-1\right)}}{Z_n(\hat{\phi_1})}$, and $ \hftwo{i}{1}= \frac{\hat{\phi_2}^{\left(\hat{\pi_2}^{-1}(\el{i})-1\right)}}{Z_n(\hat{\phi_2})}$. If $\el{i}$ does not appear in $\hat{\pi_1}$ set $\hfone{i}{1} = 0$. Similarly, if $\el{i}$ does not appear in $\hat{\pi_2}$ set $\hftwo{i}{1} = 0$. Define $g(n,\phi) = C. \frac{n^2 \phi^2}{(1-\phi)^2} \log n$, where $C$ is an absolute constant.
\label{alg:recover-rest-check-1}
\item For each $\el{i} \in \hat{\pi_1}(1:r_1/2)$
\begin{enumerate}
\item If $\hftwo{i}{1} < \frac{\min\set{\hat{w_1},\hat{w_2}}}{16} \frac{\hfone{i}{1}}{n^2g(n,\hat{\phi_1})}$
\begin{enumerate}
\item $\hat{\pi_1} \leftarrow$ {\sc Learn-Single-Mallow(${\cal{S}}_{\el{i} \mapsto 1}$)}. Here ${\cal{S}}_{\el{i} \mapsto 1}$ refers to the samples in $\cal{S}$ where $\el{i}$ goes to position $1$.   
\item $\hat{\pi_2} \leftarrow$ {\sc Find-Pi($\cal{S}$, $\hat{\pi_1}$, $\hat{w_1}$, $\hat{w_2}$, $\hat{\phi_1}$, $\hat{\phi_2}$)}. Output {\sc Success} and return $\hat{\pi_1}$ and $\hat{\pi_2}$, $\hat{w_1}$, $\hat{w_2}$, $\hat{\phi_1}$ and $\hat{\phi_2}$.
\end{enumerate}
\end{enumerate}
\item Do similar check for each $\el{i} \in \hat{\pi_2}(1:r_2/2)$.
\label{alg:recover-rest-check-2}
\item Let $e_{i^*}$ be the first element in $\hat{\pi_1}$ such that $|\hfone{i^*}{1} - \hftwo{i^*}{1}| > \epsilon$. Define $\hat{w}^{\prime}_1 = \frac 1 {1 + \frac{\hat{w_2}}{\hat{w_1}}\frac{\hftwo{i^*}{1}}{\hfone{i^*}{1}}}$ and $\hat{w}^{\prime}_2 = 1 - \hat{w}^{\prime}_1$.
\label{alg:recover-rest-find-xstar}
\item For each $\el{i} \notin \hat{\pi_{1}}$ and $j > r_1$
\label{alg:recover-rest-equations}
\begin{enumerate}
\item Estimate $\hf{i}{j} = Pr[\textrm{$\el{i}$ goes to position $j$}]$ and $\chf{i}{j}{i^*} = Pr[\textrm{$\el{i}$ goes to position $j$} | e_{i^*} \mapsto 1]$.
\item Solve the system 
\begin{eqnarray}
\hf{i}{j} & = & \hat{w_1}\hfone{i}{j} + \hat{w_2}\hftwo{i}{j}\\
\chf{i}{j}{i^*} & = & \hat{w}'_1\hfone{i}{j} + \hat{w}'_2\hftwo{i}{j}
\end{eqnarray}
\end{enumerate}

\item Form the ranking $\hat{\pi_1} = \hat{\pi_{1}} \circ \pi'_1$ s.t. for each $\el{i} \notin \hat{\pi_1} $, $pos(\el{i}) = \arg\max_{j > r_1} \hfone{i}{j}$.
\label{alg:recover-rest-mode}
\item $\hat{\pi_2} \leftarrow ${\sc Find-Pi($\cal{S}$, $\hat{\pi_1}$, $\hat{w_1}$, $\hat{w_2}$, $\hat{\phi_1}$, $\hat{\phi_2}$, $\epsilon$)}. Output {\sc Success} and return $\hat{\pi_1}$ and $\hat{\pi_2}$, $\hat{w_1}$, $\hat{w_2}$, $\hat{\phi_1}$ and $\hat{\phi_2}$.
\end{enumerate}
\end{fragment*}

\begin{fragment*}[t]
\caption{\label{alg:singlemallow}{\sc Learn-Single-Mallow}, \textbf{Input: } a set $\cal{S}$ of $N$ samples from $\cal{M}({\phi}, {\pi})$.  \vspace*{0.01in}}
\begin{enumerate} \itemsep 0pt
\small 
\item For each element $\el{i}$, estimate ${\hfone{i}{j}} = Pr[\textrm{$\el{i}$ goes to position $j$}]$.
\item Output a ranking $\hat{\pi}$ such that for all $\el{i}$, $pos(\el{i}) = \arg\max_{j} \hfone{i}{j}$.
\end{enumerate}
\end{fragment*}

\begin{fragment*}[t]
\caption{\label{alg:estimatephi}{\sc Estimate-Phi}, \textbf{Input: } $\cP$.  \vspace*{0.01in}
}
\begin{enumerate} \itemsep 0pt
\small 
\item Sort $P$ in decreasing order. Return $\min_{i} \{\frac{P_{i+1}}{P_i}\}$.
\end{enumerate}
\end{fragment*}

\begin{fragment*}[t]
\caption{\label{alg:findpi}{\sc Find-Pi}, \textbf{Input: } a set $\cal{S}$ of $N$ elements from $\cMMix$, $\hat{\pi_1}$, $\hat{w_1}$, $\hat{w_2}$, $\hat{\phi_1}$, $\hat{\phi_2}$.  \vspace*{0.01in}
}
\begin{enumerate} \itemsep 0pt
\small 
\item Compute $\hfone{i}{j} = \Prob{\textrm{$\el{i}$ goes to position $j$} | \hat{\pi_1}}$ (see Lemma~\ref{lem:find-pi}). 
\item For each element $\el{i}$, estimate ${\hat{f}_{\el{i},j}} = \Prob{\textrm{$\el{i}$ goes to position $j$}}$.
\item Solve for $\hfone{i}{j}$ using the equation $\hat{f}_{\el{i},j}  =  \hat{w_1}\hfone{i}{j} + \hat{w_2}\hftwo{i}{j}$.
\item Output $\hat{\pi_2}$ such that for each $\el{i}$, $pos(\el{i}) = \arg\max_{j} \hftwo{i}{j}$.
\end{enumerate}
\end{fragment*}

Let $\fone{i}{j}$ be the probability that element $\el{i}$ goes to position $j$ according to Mallows Model $\calM_1$ (and similarly $\ftwo{i}{j}$ for model $\calM_2$). 
To find the complete rankings, we measure appropriate statistics to set up a system of linear equations to calculate $\fone{i}{j}$ and $\ftwo{i}{j}$ up to inverse polynomial accuracy. The largest of these values $\set{\fone{i}{j}}$ corresponds to the position of $\el{i}$ in the central ranking of $\Mal_1$. To compute these values $\set{\fsup{i}{j}{r}}_{r=1,2}$ we consider statistics of the form ``\emph{what is the probability that $\el{i}$ goes to position $j$ conditioned on $\el{i^*}$ going to the first position}?''. This statistic is related to $\fone{i}{j}, \ftwo{i}{j}$ for element $\el{i^*}$ that is much closer than $\el{i}$ to the front of one of the permutations. \\

\noindent\textbf{Notation:}
Let $\fsub{i}{j}{\calM}$ be the probability that element $\el{i}$ goes to position $j$ according to Mallows Model $\calM$, and let $\fsup{i}{j}{r}$ be the same probability for the Mallows model $\calM_r$ ($r \in \set{1,2}$). Let $\cfone{i}{j}{i^*}$ be the probability that $\el{i}$ goes to the $j$th position conditioned on the element $\el{i^*}$ going to the first position according to $\calM_1$ (similarly $\Mal_2$). Finally for any Mallows model $\Mal{\phi}{\pi}$, and any element $\el{i^*} \in pi$, let $\calM_{-i^*}$ represent the Mallows model on $n-1$ elements $\Mal{\phi}{\pi - i^*}$.

In the notation defined above, we have that for any elements $\el{i^*}, \el{i}$ and position $j$, we have 
\begin{align*}
\Prob{\el{i} \rightarrow j \vert \el{i^*} \rightarrow 1}&= w'_1 \cfone{i}{j}{i^*}+ w'_2 \cftwo{i}{j}{i^*} \\
\text{where } w'_1&=\frac{w_1 x_{i*}}{w_1 x_{i^*}+w_2 y_{i^*}}, w'_2=1-w'_1
\end{align*} 

However, these statistics are not in terms of the unknown variables $\fone{i}{j}, \ftwo{i}{j}$. The following lemma shows that these statistics are \emph{almost} linear equations in the unknowns\\ $\fone{i}{j}, \ftwo{i}{j}$ for the $i,j$ pairs that we care about. For threshold $\delta$, let $r_1$ be the smallest number $r$ such that $\delta> \phi_1^{r-1}/\Zn(\phi_1)$. Similarly let $r_2$ be the corresponding number for second Mallows models $\calM_2$.
\begin{lemma}\label{lem:linear-equations-correct}
For any $j>r_1$, any elements $\el{i^*}, \el{i}$ with $\posgen_{\pi_1}(i^*)> r_1$, $\posgen_{\pi_1}(i) > \posgen_{\pi_1}(i^*)$, we have in the notation defined above that 
$$\cfone{i}{j}{i^*} = \fone{i}{j} +\delta' \quad \text{ where } \abs{\delta'} \le \delta n.$$
The corresponding statement also holds for Mallows model $\calM_2$. 
\end{lemma}
\begin{proof}
When samples are generated according to Mallows model $\calM_1$, we have for these sets of $i,i^*,j$ that the conditional probability $\cfone{i}{j}{i^*}=\fsub{i}{j-1}{\Mal{\phi_1}{\pi_1 - i^*}}$, where the term on the right is a Mallows model over $n-1$ elements. 
\begin{align*}
\fone{i}{j} = \sum_{i'=1}^{n} \Prob{\el{i'} \rightarrow 1} \cfone{i}{j}{i^*} &\le \sum_{i'=1}^{r_1} \Prob{\el{i'} \rightarrow 1} \fsub{i}{j-1}{\Mal{\phi_1}{\pi_1 - i'}}+\delta \\
&=\fsub{i}{j-1}{\Mal{\phi_1}{\pi_1 - i^*}}\sum_{i'=1}^{r_1} \Prob{\el{i'} \rightarrow 1} .
\end{align*}
The last equality is because the probability is independent of $i'$ (since $\posgen_{\pi_1}(\el{i}) > \posgen_{\pi_1}(\el{i^*})$).
Hence, it follows easily that
$$\cfone{i}{j}{i^*} (1-\delta) \le \fone{i}{j}\le \cfone{i}{j}{i^*}+ \delta.$$ 
\end{proof}

%

Hence, by picking an appropriate element $\el{i^*}$, we can set up a system of linear equations and solves for the quantities $\set{\fone{i}{j}, \ftwo{i}{j}}$. Suppose there exists an element $\el{i^*}$ that occurs in the top few positions in both the permutations, then that element would suffice for our purpose. On the other hand, if we condition on an element $i^*$ which occurs near the top in one permutation but far away in the other permutation, gives us a \emph{single} Mallows model. The sub-routine {\sc Recover-Rest} of the main algorithm figures out which of the cases we are in, and succeeds in recovering the entire permutations $\pi_1$ and $\pi_2$ in the case that $\MMix$ is non-degenerate (the degenerate cases have been handled separately in the previous section). In such a scenario, from the guarantee of Lemma~\ref{lem:fullrank} we can assume that we have parameters $\{\cw_1, \cw_2, \cphi_1, \cphi_2\}$ which are within $\eps \le \eps_0$ of the true parameters. For the rest of this section we will assume that {\sc Recover-Rest} and every sub-routine it uses has access to samples from $\cMMix$. This is w.l.o.g. due to Lemma~\ref{lem:simulate-noisy-oracle}.

The rankings $\cpi_1$ and $\cpi_2$ are obtained from {\sc Infer-Top-k}. Define $\gamma = \frac{(1-\phimax)^2}{4n\phimax}$. By our choice of $\epsilon_0$, rankings $|\hat{\pi_1}| = r_1 \ge \log_{1/\phi_1}\left(\frac{n^{10} \Zn(\phi_1)}{\wmin^2 \gamma^2} \right)$ and $|\hat{\pi_2}| = r_2 \ge \log_{1/\phi_2} \left(\frac{n^{10} \Zn(\phi_2)}{\wmin^2 \gamma^2}\right)$. We note that the values $\fone{i}{j}, \ftwo{i}{j}$ in the following Lemma are defined with respect to $\cMMix$.
 
\begin{lemma}
\label{lem:recover-rest}
Given access to an oracle for $\ccalM$ and rankings $\cpi_1$ and $\cpi_2$ which agree with $\pi_1$ and $\pi_2$ in the first $r_1$ and $r_2$ elements respectively, where $r_1 \ge \log_{1/\phi_1}\left(\frac{n^{10} \Zn(\phi_1)}{\wmin^2 \gamma^2} \right)$ and $r_2 \ge {\log_{1/\phi_2} \left(\frac{n^{10}\Zn(\phi_2)}{\wmin^2 \gamma^2}\right)}$, then procedure \textsc{Recover-Rest} with $\epsilon = \frac 1 {10} \gamma$, outputs the rankings $\pi_1$ and $\pi_2$ with high probability.
\end{lemma}
\begin{proof}
First suppose that the condition in Step~\ref{alg:recover-rest-check-1} of \textsf{Recover-Rest} is true for some $e_{i^*}$. This would imply that $\ftwo{i^*}{1} < \frac{{\cw_1}}{\hat{w_2}} \frac{\fone{i^*}{1}}{n^2g(n,\hat{\phi_1})}$. Hence, conditioned on $\el{i^*}$ going to the first position, the new weight $w'_1$ would be $\frac{1}{1+\frac{\cw_2}{\cw_1}\frac{\ftwo{i^*}{1}}{\fone{i^*}{1}}} \ge 1-\frac{1}{ng(n,\cphi_1)}$. Since, $g(n,\cphi_1)$ is an upper bound on the sample complexity of learning a single Mallows model with parameter $\cphi_1$, with high probability we will only see samples from $\pi_1$ and from the guarantees of Lemma~\ref{lem:learn-single-mallow} and Lemma~\ref{lem:find-pi}, we will recover both the permutations. A similar analysis is also true for step~\ref{alg:recover-rest-check-2} of {\sc Recover-Rest}.
If none of the above conditions happen, then step~\ref{alg:recover-rest-find-xstar} will succeed because of the guarantee from Lemma~\ref{lem:fullrank}. 

Next we will argue about the correctness of the linear equations in step~\ref{alg:recover-rest-equations}. We have set a threshold $\delta=\frac{\wmin \gamma^2}{n^4}$, from Lemma~\ref{lem:linear-equations-correct}, we know that the linear equations are correct up to error $\delta$. 
Once we have obtained good estimates for $\fone{i}{j}$ for all $\el{i}$ and $j>r$, Lemma~\ref{lem:prob-gain} implies that step~\ref{alg:recover-rest-mode} of \textsc{Recover-Rest} will give us the correct ranking $\pi_1$. This combined with Lemma~\ref{lem:find-pi} will recover both the rankings with high probability.
\end{proof}

We now present the Lemmas needed in the proof of the previous Lemma~\ref{lem:recover-rest}.
\begin{lemma}
\label{lem:prob-gain}
Consider a length $n$ Mallows  model with parameter $\phi$. Consider an element $\el{i}$ and let $pos(\el{i}) = j$. Let $\ftotal{i}{k} = Pr[\el{i} \mapsto k]$. Then we have
\begin{enumerate}
\item $\ftotal{i}{k}$ is maximum at $k = j$. 
\item For all $k > j$, $\ftotal{i}{k-1} \ge \ftotal{i}{k}(1 + gain(\phi))$.
\item For all $k < j$, $\ftotal{i}{k} \ge \ftotal{i}{k-1}(1 + gain(\phi))$.
\end{enumerate} 
Here $gain(\phi) = \frac {(1-\phi)}{4\phi} min(\frac 1 n , 1-\phi^2)$.
\end{lemma}
\begin{proof}
The case $j=1$ is easy. Let $j > 1$ and consider the case $k > j$. Let $S_k = \{\pi: pos_{\pi}(\el{i}) = k\}$. Similarly let $S_{k-1} = \{\pi: pos_{\pi}(\el{i}) = k-1 \}$. For a set $U$ of rankings, let $p(U) = Pr[\pi \in U]$. Notice that $\ftotal{i}{k-1} = p(S_{k-1})$ and $\ftotal{i}{k} = p(S_k)$.
Let $X = \{e_j: pos_{\pi^*}(e_j) > pos_{\pi^*}(\el{i})\}$ and $Y = \{e_j: pos_{\pi^*}(e_j) < pos_{\pi^*}(\el{i})\}$. We will divide $S_k$ into $4$ subsets depending on the elements $\tau_1$ and $\tau_2$ which appear in positions $(k-1)$ and $(k-2)$ respectively. In each case we will also present a bijection to the rankings in $S_{k-1}$.
\begin{itemize}
\item $S_{k,1} = \{\pi \in S_k: \tau_1, \tau_2 \in X \}$. For each such ranking in $S_k$ we form a ranking in $S_{k-1}$ by swapping $\el{i}$ and $\tau_1$. Call the corresponding subset of $S_{k-1}$ as $S_{k-1,1}$.
\item $S_{k,2} = \{\pi \in S_k: \tau_1 \in X, \tau_2 \in Y \}$. For each such ranking in $S_k$ we form a ranking in $S_{k-1}$ by swapping $\el{i}$ and $\tau_1$. Call the corresponding subset of $S_{k-1}$ as $S_{k-1,2}$.
\item $S_{k,3} = \{\pi \in S_k: \tau_1 \in Y, \tau_2 \in X \}$. For each such ranking in $S_k$ we form a ranking in $S_{k-1}$ by swapping $\el{i}$ and $\tau_1$. Call the corresponding subset of $S_{k-1}$ as $S_{k-1,3}$.
\item $S_{k,4} = \{\pi \in S_k: \tau_1, \tau_2 \in Y \}$. Consider a particular ranking $\pi$ in $S_{k,4}$. Notice that since $\el{i}$ is not in it's intended position there must exist at least one element $x \in X$ such that $pos_{\pi}(x) < pos_{\pi}(\el{i})$ in $S_k$. Let $x^*$ be such an element with the largest value of $pos_{\pi}(x)$. Let $y \in Y$ be the element in the position $pos_{\pi}(x^*) + 1$. For each such ranking in $S_k$ we form a ranking in $S_{k-1}$ by swapping $\el{i}$ and $\tau_1$ and $x^*$ and $y$.
Call the corresponding subset of $S_{k-1}$ as $S_{k-1,4}$.
\end{itemize}
It is easy to see that the above construction gives a bijection from $S_k$ to $S_{k-1}$. We also have the following
\begin{itemize}
\item $p(S_{k-1,1}) = \frac 1 {\phi} p(S_{k,1})$. This is because the swap is decreasing the number of inversions by exactly $1$.
\item $p(S_{k-1,2}) = \frac 1 {\phi} p(S_{k,2})$. This is because the swap is decreasing the number of inversions by exactly $1$.
$p(S_{k-1,3}) = \phi p(S_{k,3})$. This is because the swap is increasing the number of inversions by exactly $1$.
$p(S_{k-1,4}) = p(S_{k,4})$. This is because the two swaps maintain the number of inversions.
\end{itemize}
Also note that there is a bijection between $S_{k,2}$ and $S_{k,3}$ such that every ranking in $S_{k,3}$ has one more inversion than the corresponding ranking in $S_{k,2}$. Hence we have $p(S_{k,3}) = \phi p(S_{k,2})$.

Now we have
\begin{eqnarray}
\ftotal{i}{k-1} & = & \sum_i p(S_{k-1,i})\\
 & = & \frac 1 {\phi} p(S_{k,1}) + \frac 1 {\phi} p(S_{k,2}) + {\phi} p(S_{k,3}) + p(S_{k,4})\\
 & = & 	\ftotal{i}{k} + p(S_{k,1})(\frac 1 {\phi} - 1) + p(S_{k,2})(\frac 1 {\phi} - 1) - p(S_{k,3})(1-\phi)\\
\end{eqnarray}
If $p(S_{k,1}) \ge \frac 1 4 p(S_k)$ or $p(S_{k,2}) \ge \frac 1 4 p(S_k)$, then $gain(\phi) \ge \frac{(1-\phi)}{4\phi}(1-\phi^2)$.  If not, then we have $p(S_{k,4}) \ge 1/4$. Divide $S_{k,4}$ as $\cup_j S_{k,4,j}$ where $S_{k,4,j} = \{\pi \in S_{k,4}: pos_{\pi}(x^*) = j\}$. It is easy to see that $p(S_{k,4,j}) = \phi(S_{k,4,j-1})$. Hence we have $p(S_{k,2}) > p(S_{k,3}) > \frac 1 n p(S_{k,4}) \ge \frac 1 {4n}$. In this case we will have $gain(\phi) \ge \frac{(1-\phi)}{4n\phi}(1-\phi^2)$.

The case $k < j$ is symmetric.
\end{proof}

\begin{lemma}
\label{lem:i-to-j}
Consider a length $n$ Mallows  model with parameter $\phi$. Let the target ranking be $\pi^* = (\el{1},\el{2},\dots,\el{n})$. Let $\fone{i}{j}$ be the probability that the element at position $i$ goes to position $j$. We have for all $i,j$
\[
\fone{i}{j} = \fone{j}{i}
\]
\end{lemma}
\begin{proof}
We will prove the statement by induction on $n$. For $n=1,2$, the statement is true for all $\phi$. Now assume it is true for all $n \le l-1$. Consider a length $l$ Mallows  model. We have
\begin{align*}
\fone{i}{j}  = & \sum_{k \le i} \cfone{i-1}{j-1}{k}Pr(\el{k} \mapsto 1) + \sum_{j \ge k > i} \cfone{i}{j-1}{k}Pr(\el{k} \mapsto 1)\\ 
&\quad + \sum_{k > j} \cfone{i}{j}{k}Pr(\el{k} \mapsto 1)\\
=& \sum_{k \le i} \cfone{j-1}{i-1}{k}Pr(\el{k} \mapsto 1) + \sum_{j \ge k > i} \cfone{j-1}{i}{k}Pr(\el{k} \mapsto 1) \\
&\quad + \sum_{k > j} \cfone{j}{i}{k}Pr(\el{k} \mapsto 1)\\  
 = & \fone{j}{i}
\end{align*}
\end{proof}

\begin{lemma}
\label{lem:prob-gain-2}
Consider a length $n$ Mallows  model with parameter $\phi$. Let the target ranking be $\pi^* = (e_1, e_2, \ldots, e_n)$. Consider a position $i$ which has element $\el{i}$.
\begin{enumerate}
\item $\ftotal{j}{i}$ is maximum at $j = i$. 
\item For all $k > i$, $\ftotal{k-1}{i} \ge \ftotal{k}{i}(1 + gain(\phi))$.
\item For all $k < i$, $\ftotal{k}{i} \ge \ftotal{k-1}{i}(1 + gain(\phi))$.
\end{enumerate} 
Here $gain(\phi) = \frac {(1-\phi)}{4\phi} min(\frac 1 n , 1-\phi^2)$.
\end{lemma}
\begin{proof}
Follows from Lemmas~\ref{lem:prob-gain} and ~\ref{lem:i-to-j}.
\end{proof}

\begin{lemma}
\label{lem:remove-common-prefix}
Given access to $m = O(\frac{1}{\text{gain}(\phi)^2} \log(\frac{n}{\delta}))$ samples $\MMix$, with $\phi_1 = \phi_2$, procedure {\sc Remove-Common-Prefix} with $\epsilon = \frac 1 {10} \text{gain}(\phi)$, succeeds with probability $1-\delta$.
\end{lemma}
\begin{proof}
If the two permutations have the same first element $e_1$, then we have $x_1 = 1/Z_n(\phi)$. Since $m$ is large enough, all our estimates will be correct up to multiplicative error of $\sqrt{1+gain(\phi)}$. By induction, assume that the two permutations have the same prefix till $t-1$. By the property of the Mallows model, we know that the remaining permutations are also a mixture of two Mallows models with the same weight. Hence, at step $t$, if we estimate each probability within multiplicative factor of $\sqrt{1+gain(\phi)}$, we will succeed with high probability.
\end{proof}

\begin{lemma}
\label{lem:learn-single-mallow}
Given access to $m = O(\frac{1}{\text{gain}(\phi)^2} \log(\frac{n}{\delta}))$ samples from a Mallows  model $\cal{M}(\phi,\pi)$, procedure \textsc{Learn-Single-Mallow} with $\epsilon = \frac 1 {10} \text{gain}(\phi)$, succeeds with probability $1-\delta$.
\end{lemma}
\begin{proof}
In order to learn, it is enough to estimate $\ftotal{i}{j} = Pr[\el{i} \text{ goes to position } j]$ for every element $\el{i}$ and position $j$. Having done that we can simply assign $\text{pos}(\el{i}) = \arg\max_j \ftotal{i}{j}$. From Lemma~\ref{lem:prob-gain} we know that this probability is maximum at the true location of $\el{i}$ and hence is at least $1/n$. Hence, it is enough to estimate all $\ftotal{i}{j}$ which are larger than $1/n$ up to multiplicative error of $\sqrt{1+\text{gain}(\phi)}$. By standard Chernoff bounds, it is enough to sample $O(\frac{1}{\text{gain}(\phi)^2} \log(\frac{n}{\delta}))$ from the oracle for $\cal{M}(\phi,\pi)$.  
\end{proof}
\begin{lemma}
\label{lem:find-pi}
Given the parameters of a mixture model $\MMix$ and one of the permutations $\pi_1$, procedure \textsf{Find-Pi} with $\epsilon = \frac {\wmin \gamma} {10}$, succeeds with probability $1-\delta$. Here $\gamma = \min(\text{gain}(\phi_1), \text{gain}(\phi_2))$.
\end{lemma}
\begin{proof}
For any element $\el{i}$ and position $j$, we have that 
\begin{equation}
\label{eq:f-i-to-j}
\ftotal{i}{j} = w_1 \fone{i}{j} + w_2 \ftwo{i}{j}.
\end{equation}
Here $\fone{i}{j}$ is the probability that element $\el{i}$ goes to position $j$ in $\calM(\phi_1, \pi_1)$. Similarly, $\ftwo{i}{j}$ is the probability that element $\el{i}$ goes to position $j$ in $\calM(\phi_2, \pi_2)$. We can compute $\fone{i}{j}=f^{(1)}_{(n,j,i)}$ using dynamic programming via the following relation
\begin{eqnarray*}
&& f^{(1)}_{(n,1,i)} = \phi_1^{i-1} / Z_n(\phi_1) \cr
&& f^{(1)}_{(n,l,i)} = \tfrac 1 {Z_n(\phi_1)} \left[\left(\sum_{j=1}^{i-1} \phi_1^{j-1}\right)f^{(1)}_{(n-1, l-1,i-1)}  + \left(\sum_{j=i+1}^{n} \phi_1^{j-1}\right)f^{(1)}_{(n-1, l-1,1)}  \right]
\end{eqnarray*}
Here $f^{(1)}_{(n,l,i)}$ is the probability that the element at the $i$th position goes to position $l$ in a length $n$ Mallows  model. Notice that this probability is independent of the underlying permutation $\pi$. Having computed $\fone{i}{j}$ using the above formula, we can solve Equation~\ref{eq:f-i-to-j} to get $\ftwo{i}{j}$ to accuracy $\sqrt{1+\wmin\gamma}$ and figure out $\pi_2$. The total number of samples required will be $O(\frac{1}{\gamma^2 \wmin^2} \log(\frac{n}{\delta}))$.
\end{proof}

\section{Wrapping up the Proof}
\begin{proof}[Proof of Theorem~\ref{thm:main}]
\input{wrapping}
\end{proof}
\section{Conclusions and Future Directions}
\input{conclusions}
\label{sec:future_directions}


\section{Some Useful Lemmas for Error Analysis}

\newcommand{\up}{u^{\perp}}
\newcommand{\vp}{v^{\perp}}
\begin{lemma}\label{lem:error1}
Let $u,u',v,v'$ denote vectors and fix parameters $\delta,\gamma>0$.
Suppose $\norm{u \otimes v - u' \otimes v'}_{F} < \delta$, and $\gamma \le \norm{u}, \norm{v}, \norm{u'}, \norm{v'} \le 1$,\\ with $\delta <\frac{\gamma^2}{2}$. Given a decomposition $u=\alpha_1 u'+ \up$ and $v=\alpha_2 v'+\vp$, where $\up$ and $\vp$ are orthogonal to $u', v'$ respectively, then we have
$$ \norm{\up} < \sqrt{\delta} ~ \text{and} ~ \norm{\vp} < \sqrt{\delta}.$$
\end{lemma}
\begin{proof}
We are given that $u=\alpha_1 u' +\up$ and $v=\alpha_2 v' + \vp$. Now, since the tensored vectors are close

\begin{align}
\norm{u \otimes v - u' \otimes v'}_{F}^2 &< \delta^2 \notag\\  
\norm{(1-\alpha_1 \alpha_2) u'\otimes v' + \alpha_2 \up \otimes v'+ \alpha_1 u' \otimes \vp + \up \otimes \vp}_{F}^2 &< \delta^2 \notag\\
\gamma^4 (1-\alpha_1 \alpha_2)^2 + \norm{\up}^2 \alpha_2^2 \minl^2 +\norm{\vp}^2 \alpha_1 ^2 \gamma^2 + \norm{\up}^2 \norm{\vp}^2 &< \delta^2 \label{eq:tensoring:lb}
\end{align}
This implies that $|1-\alpha_1 \alpha_2| < \delta/ \gamma^2$.

Now, let us assume $\beta_1=\norm{\up}> \sqrt{\delta}$. This at once implies that $\beta_2=\norm{\vp} < \sqrt{\delta}$. Hence one of the two (say $\beta_2$) is smaller than $\sqrt{\delta}$.\\
Also 
\begin{align*}
\gamma^2 \le \norm{v}^2 &= \alpha_2^2 \norm{v'}^2 + \beta_2^2  \\
\gamma^2 -\delta &\le \alpha_2^2 \\
\text{ Hence, }\quad \alpha_2 &\ge \frac{\gamma}{2} 
\end{align*}
Now, using \eqref{eq:tensoring:lb}, we see that $\beta_1 < \sqrt{\delta}$.
\end{proof}

\begin{lemma}\label{lem:boundingc}
Let $\phi \in (0,1)$ be a parameter and denote $c_2(\phi)=\frac{\Zpart{n}(\phi)}{\Zpart{n-1}(\phi)}\frac{1+\phi}{\phi}$ and $c_3(\phi)=\frac{\Zpart{n}^2(\phi)}{\Zpart{n-1}(\phi)\Zpart{n-2}(\phi)}\frac{1+2\phi+2\phi^2+\phi^3}{\phi^3}$. Then we have that $1 \le c_2(\phi)\le 3/\phi$ and $1\le c_3(\phi)\le 50/\phi^3$.
\end{lemma}
\begin{proof}
Since $0<\phi <1$, we have that $\Zpart{n-1}(\phi) \le \frac{1}{1-\phi}$. Observe that $1 \le \frac{\Zpart{n}(\phi)}{\Zpart{n-1}(\phi)} \le 1+\frac{1}{\Zpart{n-1}(\phi)} \le 2$. The bounds now follow immediately. 
\end{proof}

\begin{lemma} \label{lem:sampling}
In the notation of section~\ref{sec:prelims}, given $N$ independent samples, the empirical average $\cP$ satisfied $\norm{P-\cP}_\infty < \sqrt{C\frac{\log n}{N}}$ with probability $1-n^{-C/8}$.  
\end{lemma}
\begin{proof}
This follows from a standard application of Bernstein inequality followed by a union bound over the $O(n^3)$ events. 
\end{proof}

%% file: degenerate.tex

\begin{lemma}[Staggered degenerate case]\label{lem:staggered}
Suppose $\phi=\phi_1=\phi_2$, and at most two of the top elements $L_{\sqrt{\eps}}$ are not in bucket $B_{\ell^*}$  i.e. $B_{\ell^*} \ge L_{\sqrt{\eps}} -2 $ with $\ell^* \ne 0$. Then, for any $\eps>0$, given $N > \text{poly}(n,\phi,\eps,\wmin)$ samples, step (3-4) of Algorithm {\sc Handle-Degenerate} finds finds $\cw_1, \cw_2$ of $w_1, w_2$ up to $\eps$ accuracy and the top $m$ elements of $\pi_1, \pi_2$ respectively where $m=\frac{\log \Zn(\eps)}{2\log \phi}$.  
\end{lemma}
\begin{proof}
Since $\phi_1=\phi_2$, we can use Lemma~\ref{app_lem:simulate-new-element}, we can sample from a Mallows mixture where we add one new element $\el{3}^*$ to the front of both permutations $\pi_1, \pi_2$. Doing this two more times we can sample from a Mallows mixture where we add $\el{1}^*,\el{2}^*, \el{3}^*$ to the front of both permutations. Let these new concatenated permutations be $\pi_1^*, \pi_2^*$. Since the majority bucket corresponds to $\ell^* \ne 0$, we have at least three pairs of elements which satisfy Lemma~\ref{lem:bucket1}, we see that w.h.p. in one of the $O(\log n)$ iterations, the partitions $S_a, S_b, S_c$ have $\sigma_2(\cdot) \ge (\minl \phi^6)^2 (1-\phi)^3$. 

Hence, by using the Tensor algorithm with guarantees from Lemma~\ref{lem:fullrank}, and using Algorithm {\sc Recover-Rest}, we  get the full rankings as required (using Lemma~\ref{lem:recover-rest}).
\end{proof}

\begin{lemma}[Aligned Degenerate case]\label{lem:aligned}
Suppose $\phi=\phi_1=\phi_2$, and at most two of the top elements $L_{\sqrt{\eps}}$ are not in bucket $B_{0}$  i.e. $\abs{B_{0}} \ge \abs{L_{\sqrt{\eps}}} -2 $. For any $\eps>0$, given $N = O(\frac{n^2 \log n}{\eps^8 \wmin^2 (1-\phi)^4})$ samples,  steps (5-10) of Algorithm~\ref{alg:degenerate} ({\sc Handle-Degenerate}) finds estimates $\cw_1, \cw_2$ up to $\eps$ accuracy and prefixes $\pi_1', \pi_2'$ of $\pi_1, \pi_2$ respectively that contain at least the top $m$ elements where $m=\frac{\log \Zn(\eps)}{2\log \phi}$.  
\end{lemma}
\begin{proof}
The first position differs because of step(1-2) of Algorithm~\ref{alg:degenerate}. Without loss of generality $\pi_1=\pi^{pfx}_1$ and $\pi_2=\pi^{pfx}_2$. $B_{0} \ge m-2$, hence $\abs{B_{0}} =m-2$. Let $\el{i_1}, \el{j_1}$ be the other two elements in $L_{\sqrt{\eps}}$. 

For elements $\el{i} \in B_0, \pi^{-1}_1(i)=\pi^{-1}_2(i)$. The sampling error in the entries of $\cP$ is at most $\eps_s=\eps^4 \wmin (1-\phi)^2/ n$. Hence, they fall into the set $I_{\pi_1^{-1}(i)}$. Therefore, there can be at most four sets with at most four elements between them that constitute $I_{bad}$. 

Consider $M=(\cP_{ij})_{i \in S_a \cup \set{i_1}, j \in S_b \cup \set{j_1}}$. Also let $\Ma=(\xa ; \ya)$ and $\Mb=(\xa ; \ya)$ applied to $\calM'$. By Lemma~\ref{lem:bucket1}, we see that $\sigma_2(\Ma), \sigma_2(\Mb) \ge \eps (1-\phi)$. Further, 
$$ \norm{M - \Ma \left( \begin{matrix} w_1 & 0 \\ 0 & w_2 \end{matrix}\right) \Mb^T}_F \le \eps_s n.$$
Hence, $\sigma_2(M) \ge \eps^2 (1-\phi)^2 \wmin$. If $i_1, j_1$ do not belong to the two different partitions $S_a, S_b$, it is easy to see that $\sigma_2(M) \le \sqrt{\eps_s} > \eps^2 \wmin (1-\phi)^2$. Hence, we identify the two irregular elements that are not in bucket $B_0$, and use this to figure out the rest of the permutations.   
\end{proof}

Finally, the following lemma shows how the degenerate cases are handled.
\begin{lemma}\label{lem:degenerate}
For $0<\epsilon$, given $\phi_1, \phi_2$ with $\abs{\phi_1-\phi_2}\le \eps_1 =\errp_{\ref{lem:degenerate}}(n,\phi,\eps,\wmin)$, such that at most two  elements of $L_{\sqrt{\eps}}$ are not in the bucket $B_{\ell^*}$, then Algorithm {\sc Handle-Degenerate} finds w.h.p. estimates $\cw_1, \cw_2$ of $w_1, w_2$ up to $\eps$ accuracy, and recovers $\pi_1, \pi_2$.
\end{lemma}
\begin{proof}
We can just consider the case $\cphi=\cphi_1=\phi_2$ using Lemma~\ref{lem:simulate-noisy-oracle} as long as $\eps_1 < \frac{\phi}{n^2 N(n,\phi,\eps)^2}$, where $N$ is the number of samples used by Lemma~\ref{lem:aligned} and Lemma~\ref{lem:staggered} to recover the rest of the permutations and parameters up to error $\eps$. This is because the simulation oracle does not fail on any of the samples w.h.p, by a simple union bound. 

If the two permutations do not differ at all, then by Lemma~\ref{lem:remove-common-prefix}, Algorithm~\ref{alg:commonprefix} returns the whole permutation $\pi_1=\pi_2$. Further, any set of weights can be used since both are identical models ($\phi_1=\phi_2=\phi$). 

Let $m=L_{\sqrt{\eps}}$. 
In the remaining mixture $\calM'$, the first position of the two permutations differ: hence, $B_{\ell^*}< m$. 
Further, we know that $B_{\ell^*} \ge m-2$. 

We have two cases, depending on whether the majority bucket $B_{\ell^*}$ corresponds to $\ell^*=0$ or $\ell^* \ne 0$.
In the first case, Lemma~\ref{lem:staggered} shows that we find the permutations $\pi_1,\pi_2$ and parameters up to accuracy $\eps$. If this FAILS, we are in the case $\ell^*=0$, and hence Lemma~\ref{lem:aligned} shows that we find
the permutations $\pi_1,\pi_2$ and parameters up to accuracy $\eps$. 
\end{proof}

%% file: conclusions.tex
In this paper we gave the first polynomial time algorithm for learning the parameters of a mixture of two Mallows models. Our algorithm works for an arbitrary mixture and does not need separation among the underlying base rankings. We would like to point out that we can obtain substantial speed-up in the first stage (tensor decompositions) of our algorithm by reducing to an instance with just $k \sim \log_{1/\phi} n$ elements. 

Several interesting directions come out of this work. A natural next step is to generalize our results to learn a mixture of $k$ Mallows models for $k>2$. We believe that most of these techniques can be extended to design algorithms that take $\poly(n,1/\eps)^k$ time. It would also be interesting to get algorithms for learning a mixture of $k$ Mallows models which run in time $\poly(k,n)$, perhaps in an appropriate smoothed analysis setting~\cite{BCMV} or under other non-degeneracy assumptions. 
Perhaps, more importantly, our result indicates that tensor based methods which have been very popular for problems such as mixture of Gaussians, might be a powerful tool for solving learning problems over rankings as well. We would like to understand the effectiveness of such tools by applying them to other popular ranking models as well.